\documentclass[12pt,a4paper]{article}

\usepackage[english]{babel}
\usepackage[utf8]{inputenc}
\usepackage[T1]{fontenc}
\usepackage[a4paper,top=1in,bottom=1in,left=1in,right=1in]{geometry}

\usepackage{natbib}
\usepackage{setspace}
\usepackage{amsmath,amssymb,amsfonts,amsthm,mathtools}
\usepackage{algorithm,algpseudocode}
\usepackage{bm}
\usepackage{graphicx}
\usepackage{float}
\usepackage{booktabs,array,multirow}
\usepackage{caption,subcaption}
\usepackage{lmodern}
\usepackage{microtype}
\usepackage{hyperref}
\usepackage{enumitem}
\usepackage{adjustbox}
\usepackage{tikz}
\usepackage{xcolor}
\usepackage{aliascnt}
\usepackage[nameinlink,noabbrev,capitalise]{cleveref}

\hypersetup{
  colorlinks=true,
  linkcolor=blue!50!black,
  citecolor=blue!50!black,
  urlcolor=blue!60!black
}

\allowdisplaybreaks
\numberwithin{equation}{section}

\algrenewcommand\algorithmicrequire{\textbf{Input:}}
\algrenewcommand\algorithmicensure{\textbf{Output:}}

\theoremstyle{plain}
\newtheorem{theorem}{Theorem}[section]
\newaliascnt{proposition}{theorem}
\newtheorem{proposition}[proposition]{Proposition}
\aliascntresetthe{proposition}
\newaliascnt{lemma}{theorem}
\newtheorem{lemma}[lemma]{Lemma}
\aliascntresetthe{lemma}
\newaliascnt{corollary}{theorem}
\newtheorem{corollary}[corollary]{Corollary}
\aliascntresetthe{corollary}
\theoremstyle{definition}
\newaliascnt{definition}{theorem}
\newtheorem{definition}[definition]{Definition}
\aliascntresetthe{definition}
\newaliascnt{assumption}{theorem}

\aliascntresetthe{assumption}
\newaliascnt{example}{theorem}

\aliascntresetthe{example}
\theoremstyle{remark}
\newaliascnt{remark}{theorem}
\newtheorem{remark}[remark]{Remark}
\aliascntresetthe{remark}

\crefname{theorem}{Theorem}{Theorems}
\Crefname{theorem}{Theorem}{Theorems}
\crefname{proposition}{Proposition}{Propositions}
\Crefname{proposition}{Proposition}{Propositions}
\crefname{lemma}{Lemma}{Lemmas}
\Crefname{lemma}{Lemma}{Lemmas}
\crefname{corollary}{Corollary}{Corollaries}
\Crefname{corollary}{Corollary}{Corollaries}
\crefname{definition}{Definition}{Definitions}
\Crefname{definition}{Definition}{Definitions}
\crefname{assumption}{Assumption}{Assumptions}
\Crefname{assumption}{Assumption}{Assumptions}
\crefname{example}{Example}{Examples}
\Crefname{example}{Example}{Examples}
\crefname{remark}{Remark}{Remarks}
\Crefname{remark}{Remark}{Remarks}

\newcommand{\E}{\mathbb{E}}
\newcommand{\R}{\mathbb{R}}
\newcommand{\KL}{\mathrm{KL}}
\newcommand{\W}{\mathcal{W}}
\newcommand{\U}{\mathcal{U}}
\newcommand{\proxy}{\hat r}
\newcommand{\true}{r^\star}
\newcommand{\advreward}{s}
\newcommand{\bud}{\delta}
\newcommand{\argmax}{\mathop{\mathrm{arg\,max}}}
\newcommand{\argmin}{\mathop{\mathrm{arg\,min}}}
\newcommand{\ip}[2]{\left\langle #1, #2 \right\rangle}
\newcommand{\norm}[1]{\left\lVert #1 \right\rVert}
\newcommand{\ind}{\mathbf{1}}
\newcommand{\rollpi}{\bar\pi}
\newcommand{\algchange}[1]{\textcolor{blue!70!black}{#1}}
\DeclareMathOperator{\Reg}{Reg}
\renewcommand{\thefootnote}{\fnsymbol{footnote}}
\makeatletter
\providecommand{\theHALG@line}{}
\renewcommand{\theHALG@line}{\thealgorithm.\arabic{ALG@line}}
\makeatother

\title{Wasserstein Distributionally Robust Regret Optimization for Reinforcement Learning from Human Feedback}
\author{Yikai Wang$^{*,1}$, Shang Liu$^{*,2,3}$, Jose Blanchet$^{3}$}
\date{\small
$^{1}$Department of Statistics and Operations Research, University of North Carolina\\
$^{2}$Imperial Business School, Imperial College London\\
$^{3}$Department of Management Science and Engineering, Stanford University}

\begin{document}
\maketitle
\onehalfspacing

\def\thefootnote{*}\relax\footnotetext{Equal contribution. Corresponding to s.liu21@imperial.ac.uk.}
\setcounter{footnote}{0}\renewcommand{\thefootnote}{\arabic{footnote}}

\begin{abstract}
Reinforcement learning from human feedback (RLHF) is a central post-training tool for aligning large language models, but its training reward is only a learned proxy for true human utility. This creates a decision problem under objective misspecification: the policy is optimized against an estimated reward, while deployment performance is governed by an unobserved population preference. The resulting gap leads to reward over-optimization, where proxy reward keeps improving after true quality deteriorates. We propose distributionally robust regret optimization (DRRO) for RLHF with a Wasserstein ambiguity set over reward laws, using promptwise \(\ell_p\) distances between reward vectors as transport costs. Unlike standard distributionally robust optimization, which pessimizes worst-case value, DRRO pessimizes worst-case regret relative to the best policy under the same plausible reward perturbation. We show that the expressive-policy problem decomposes into promptwise regret problems. For each prompt, the inner adversary has a dual-norm closed form; under the \(\ell_1\) transport cost used by our algorithm, the optimizer has a water-filling structure. These results lead to a practical policy-gradient algorithm that adds a simple sampled bonus to GRPO-style training. Theory and experiments both show that DRRO is less over-pessimistic than standard DRO and mitigates over-optimization more effectively than existing baselines.
\end{abstract}

\section{Introduction}
\label{sec:introduction}

Large language models owe much of their practical usefulness not only to scale, but also to \emph{post-training}. Pretraining creates a strong next-token predictor, but not necessarily a model that is helpful, safe, or responsive to human intent. Modern alignment pipelines therefore add supervised fine-tuning and preference-based policy optimization, most prominently reinforcement learning from human feedback (RLHF) \citep{christiano2017deep,ziegler2019fine,stiennon2020learning,ouyang2022training}. In a typical RLHF pipeline, prompts are sampled, the model generates candidate responses, humans express preferences, a reward model is fit to those preferences, and the language model is optimized against the learned reward. The leverage of this stage is substantial: \citet{ouyang2022training} reported that a 1.3B aligned model was preferred by human raters to a 175B model without comparable RLHF.

The same pipeline also creates the problem studied in this paper. Direct human feedback is too expensive to collect in real time while policy optimization repeatedly samples massive numbers of responses. RLHF therefore trains a reward model to mimic human ratings. The Bradley--Terry comparison model is perhaps the most widely adopted approach \citep{bradley1952rank,christiano2017deep,ouyang2022training}: for two responses to the same prompt, the preference probability is modeled through the difference between their learned rewards. This construction makes the central issue transparent. The learned reward score is not human preference itself; it is a proxy, estimated from finite and distribution-limited comparisons, and then optimized as if it were the objective. Goodhart's law warns that when a measurement becomes a goal to optimize, the measurement may no longer be a good measurement \citep{goodhart1975problems,strathern1997improving}. RLHF places that warning directly inside a learning algorithm.

The empirical signature is now clear. In the controlled study of \citet{gao2023scaling}, the proxy reward continued to improve as the policy moved farther from the initial model, while a stronger gold-standard reward first peaked and then deteriorated. The very signal used for training can therefore certify progress while actual quality is declining. Related evidence appears in studies of human-feedback confounders \citep{hosking2024gold}, direct alignment over-optimization \citep{rafailov2024directoveropt}, and Goodhart effects in reinforcement learning \citep{karwowski2024goodhart}.

\begin{figure}[t]
\centering
\includegraphics[width=0.72\linewidth]{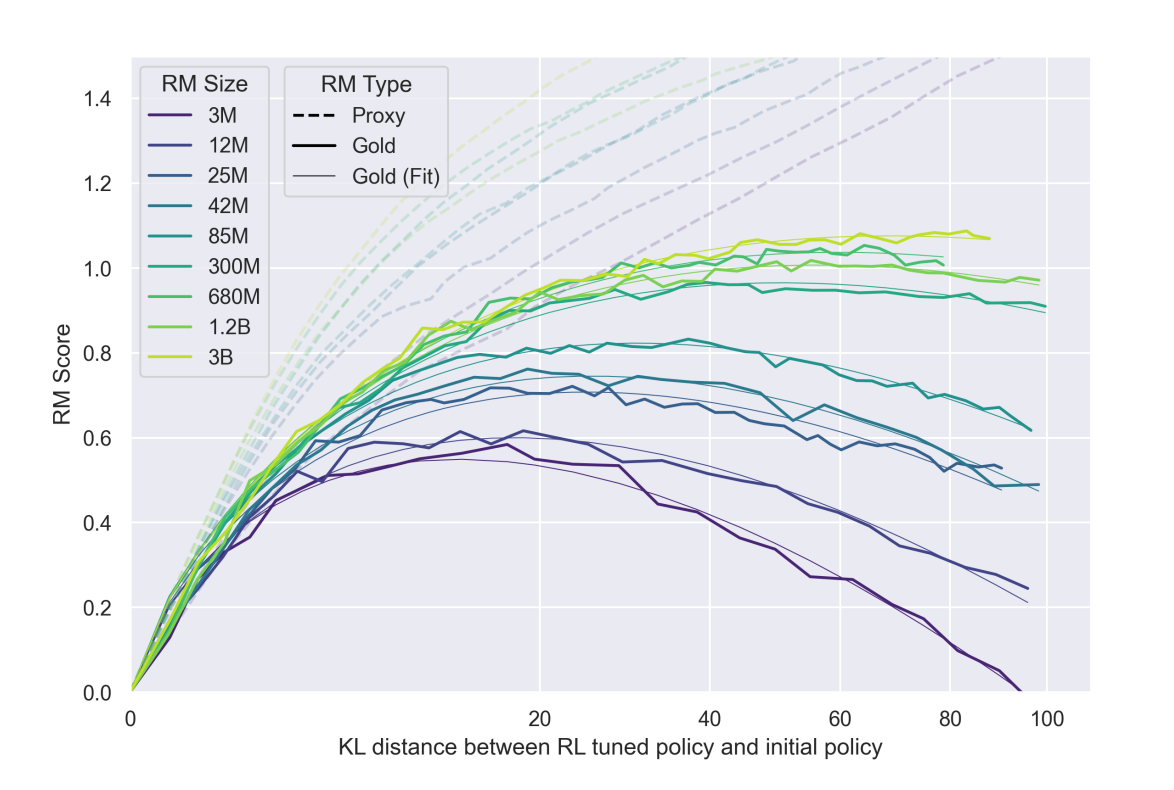}
\caption{Reward over-optimization in RLHF, adapted from \citet{gao2023scaling}. The horizontal axis measures how far policy optimization has moved from the initial policy, while the vertical axis reports reward-model evaluation. When the policy is optimized against the proxy reward model, the proxy evaluation keeps increasing, but the gold reward model's evaluation first improves and then deteriorates.}
\label{fig:intro-overoptimization}
\end{figure}

Conceptually, fix a prompt and imagine listing all complete responses. Promptwise RLHF shifts probability mass toward responses with high learned proxy reward. Deployment quality, however, is governed by population preferences: different users may assign different latent rewards to the same response, and the relevant target is their aggregate preference rather than the learned score. Downstream RLHF is therefore decision making under objective misspecification. More importantly, the misspecification is decision-dependent: optimization changes the response distribution, which changes where the reward model is asked to extrapolate.

Existing mitigations broadly take two paths. One constrains the updated policy near an initial, reference, or behavior policy, through divergence-based or reference-policy control \citep{schulman2017ppo,ouyang2022training,xiong2024iterative}, explicit constrained RLHF \citep{moskovitz2024constrained}, or behavior- and demonstration-supported regularization \citep{dai2025behavior,rita2024demonstration,liu2024provably}. These safeguards can be useful, but they also suppress potentially valuable improvement and hide the deeper question of what objective should be optimized.

The other path improves the proxy itself, for example through reward-model ensembles \citep{coste2024ensembles,eisenstein2024helping}, iterative or active preference collection \citep{xiong2024iterative,das2024active}, alternative reward-model objectives \citep{miao2024inform}, or robust reward modeling \citep{liu2025rrm,bukharin2025adversarial}. These approaches address the right object, but typically require additional data, models, and/or training stages.

A natural alternative is to treat the reward model as an uncertain objective and import robust optimization. Standard distributionally robust optimization (DRO) evaluates a policy by its worst plausible reward model inside an ambiguity set around the learned reward model, drawing on a powerful operations research toolkit for robust and Wasserstein distributionally robust optimization \citep{ben2013robust,wiesemann2014distributionally,esfahani2018data}. In our setting, the Wasserstein distance is taken over reward laws, and the transport cost is the promptwise \(\ell_p\) distance between reward vectors.

Yet worst-case value is not always the right benchmark. Pessimistic value penalties and restrictive feasible regions can become over-conservative \citep{xu2025pet,li2025geb,wolf2025iterated}; in our benchmark, standard DRO even performs worse than raw PPO in peak held-out gold reward, as detailed in the experiments. Inspired by regret in reinforcement learning \citep{jaksch2010nearoptimal,jin2018qlearning}, we argue that RLHF should not robustify absolute value alone. Here, regret means the gap between the chosen policy and the best policy that would have been chosen under the same realized plausible reward.

This motivates our central idea: robustify \emph{regret}, not value. Distributionally robust regret optimization (DRRO) evaluates a policy by its worst plausible regret against the hindsight-best policy under the same plausible reward model. This benchmark is especially natural in RLHF because one policy must be committed before the realized user preference is observed.

The obstacle is computational. Worst-case regret is generally much harder to optimize than worst-case value: classical minmax-regret problems are NP-hard even in static settings \citep{averbakh2004minmax,averbakh2005complexity}, and recent ex-ante DRRO formulations often require substantial optimization machinery \citep{agarwal2022minimax,cho2024wasserstein,bitar2024distributionally,fiechtner2025wasserstein}. This seems discouraging for RLHF, where the action space is the set of all possible language-model responses.

The main theoretical message of this paper is that promptwise RLHF has far more structure than this pessimistic view suggests. For an expressive policy class, the global ex-ante problem decomposes into promptwise regret problems. Each fixed-prompt problem becomes a portfolio-style allocation over responses. Under Wasserstein ambiguity with promptwise \(\ell_p\) transport cost, the inner adversary has an explicit dual-norm form; under the \(\ell_1\) transport cost used by our algorithm, the minimizer further collapses to a water-filling rule. This also explains why DRRO is less pessimistic than DRO: DRRO uses the proxy's ranking information when deciding how to hedge, whereas DRO under the same formulation can ignore it. As ambiguity shrinks, DRRO recovers proxy-greedy behavior; as ambiguity grows, it converges to the uniform policy.

Closed-form promptwise solutions are insightful, but not yet an RLHF algorithm. The full response simplex is hidden and astronomically large, and the trainable object is the LLM parameter vector rather than an explicit probability table. We therefore derive policy-gradient surrogates for DRRO, show that the robust correction appears as a simple reward bonus, and replace the exact maximum by a smooth maximum estimated from sampled responses through self-normalized importance sampling. With numerically stable group-normalized rollout probabilities, the resulting method requires only minimal changes to GRPO implementations \citep{shao2024deepseekmath} and gives the strongest over-optimization mitigation among the baselines we test.

\paragraph{Contributions.}
The paper makes five contributions.
\begin{enumerate}[leftmargin=1.8em]
  \item \textbf{Formulation.} We formulate RLHF under reward-model misspecification as an ex-ante DRRO problem with Wasserstein ambiguity over reward laws and promptwise \(\ell_p\) transport costs.
  \item \textbf{Promptwise theory.} We prove promptwise decomposition, derive dual-norm closed-form adversaries, and obtain a water-filling optimal policy under the \(\ell_1\) transport cost used by the algorithm.
  \item \textbf{Algorithm.} We turn the exact theory into policy-gradient surrogates, a sampling-based smooth-maximum estimator, a stable group-normalized implementation, and ambiguity-budget rules.
  \item \textbf{Experiments.} On a matched VERL benchmark, soft dynamic DRRO reaches peak held-out gold reward 1.43, compared with 1.27 for GRPO, 1.20 for PPO, and 0.79 for DRO.
  \item \textbf{DRO comparison.} We show theoretically that DRRO can be less over-pessimistic than DRO, both under rank-preserving rewards and through relative concentrability.
\end{enumerate}

The rest of the paper introduces the model, develops the promptwise theory and DRRO-GRPO algorithm, reports the experiments, and then compares DRO and DRRO theoretically. The proofs and further details are deferred to the appendix.

\section{RLHF as Ex-Ante Regret Robustness}
\label{sec:model}

This section formalizes the RLHF decision problem. We fix notation for prompts, policies, proxy rewards, and latent human preferences, then show how a Wasserstein ambiguity set over reward distributions reduces to the promptwise mean-reward ambiguity set solved in \Cref{sec:promptwise}.

Let $\mathcal X$ be the prompt space and let $\mathcal D$ be the prompt distribution used for training and deployment evaluation. For the current prompt, let $\mathcal Y$ denote the set of complete responses. This set is astronomically large but conceptually finite after fixing a vocabulary and maximum generation length; the theory uses this finite simplex view, while the algorithm only samples completions. A policy $\pi_\theta(\cdot\mid x)$ assigns probabilities to complete responses, $\pi_0$ denotes the frozen reference policy, and $\rollpi$ denotes the rollout policy used during a policy-gradient update.

Each user $u$ has a latent preference reward $\true_u(x,y)$. The user population $\mathcal Q$ induces a distribution $\mathcal P^\star$ over reward functions, while the learned reward model $\proxy(x,y)$ is only a proxy for this population signal. For any distribution $P$ over reward functions, define
\begin{equation}
\label{eq:return}
J_P(\pi)\coloneqq\E_{x\sim\mathcal D,\ y\sim\pi(\cdot\mid x),\ R\sim P}[R(x,y)].
\end{equation}
For the true population law, $J_{\mathcal P^\star}(\pi)=\E_{u\sim\mathcal Q,\ x\sim\mathcal D,\ y\sim\pi(\cdot\mid x)}[\true_u(x,y)]$. If $P=\delta_r$ is a point mass on reward function $r$, write $J_r(\pi)\coloneqq\E_{x\sim\mathcal D,\ y\sim\pi(\cdot\mid x)}[r(x,y)]$. By linearity, only the certainty-equivalent reward $\bar r_P(x,y)\coloneqq\E_{R\sim P}[R(x,y)]$ matters: $J_P(\pi)=J_{\bar r_P}(\pi)$. We set $\true\coloneqq\bar r_{\mathcal P^\star}$ and let $\Pi$ denote the policy class considered by the robust decision problem.

Nominal RLHF maximizes $J_{\proxy}(\pi)$, often with KL or clipping controls inherited from PPO-style training \citep{schulman2017ppo}. Our algorithms do not add an explicit KL penalty to the policy objective: the robust correction changes the reward signal rather than merely restricting optimization of the same proxy. KL reappears only as a drift statistic for calibrating the dynamic ambiguity budget; \Cref{app:kl-discussion} gives more detail.

The closed-form theory is promptwise. Fix $x$ and enumerate its complete responses as $y_1,\ldots,y_n$. Write $\pi_i(x)\coloneqq\pi(y_i\mid x)$ and $\beta_i(x)\coloneqq\beta(y_i\mid x)$, so $\pi(x),\beta(x)\in\Delta_n$, where $\Delta_n\coloneqq\{\zeta\in\R_+^n:\sum_i \zeta_i=1\}$. For any reward function $r$, define $r_i(x)\coloneqq r(x,y_i)$ and $r(x)\coloneqq(r_1(x),\ldots,r_n(x))\in\R^n$. Similarly, $\proxy(x)\in\R^n$ denotes the promptwise proxy-reward vector.

Let $P_0\coloneqq\delta_{\proxy}$ be the nominal point mass on the learned reward model. For $p,q\ge1$, define the promptwise ground metric
\[
d_p(\mu,\nu)\coloneqq\norm{\mu-\nu}_p,
\]
and let $\W_{q,d_p}$ be the $q$-Wasserstein distance induced by the ground metric $d_p$:
\begin{equation}
\label{eq:promptwise-wasserstein}
\begin{aligned}
\W_{q,d_p}(P,Q;x)
&\coloneqq
\left(
\inf_{\gamma\in\Gamma(P,Q)}
\E_{\gamma}\!\left[
d_p(R(x),\widetilde R(x))^q
\right]
\right)^{1/q},
\end{aligned}
\end{equation}
where $\Gamma(P,Q)$ is the set of couplings of $P$ and $Q$. The mean-reward reduction below holds for all $p,q\ge1$, so the robust model can be stated for any norm-induced transport geometry. For fixed \(p,q\), define the ambiguity set
\begin{equation}
\label{eq:distributional-ambiguity}
\U_{\bud}(P_0)\coloneqq
\left\{
P:\ \W_{q,d_p}(P,P_0;x)\le \bud\ \text{for every }x\in\mathcal X
\right\}.
\end{equation}
Here $\bud$ is the scalar ambiguity budget, not the Dirac delta in $P_0=\delta_{\proxy}$. The algorithmic section later specializes to \(p=1\), because that geometry gives the coordinate-local correction needed for sampled completions; see \Cref{sec:algorithm-discussion,prop:l1-rationale}.

The decision is ex ante: the learner chooses one policy before knowing which plausible reward distribution will govern deployment. The DRRO problem is
\begin{equation}
\label{eq:ex-ante-drro}
\min_{\pi\in\Pi}\ \max_{P\in\U_{\bud}(P_0)}
\left[
\max_{\beta\in\Pi}J_P(\beta)-J_P(\pi)
\right],
\end{equation}
where $\beta$ is the hindsight policy that would be selected if $P$ were known. This benchmark is the key distinction: DRRO compares the chosen policy with the best policy in the same plausible world, whereas DRO evaluates absolute worst-case value.

\begin{proposition}[Mean-reward reduction]
\label{prop:mean-reduction}
For any policy $\pi$ and any distribution $P$ over reward functions, define the mean reward
\begin{equation*}
\bar r_P(x,y):=\E_{R\sim P}[R(x,y)].
\end{equation*}
Then
\begin{equation*}
J_P(\pi)=J_{\bar r_P}(\pi).
\end{equation*}
Consequently, the ex-ante DRRO objective depends on $P$ only through the induced mean-reward set.
\end{proposition}

\begin{corollary}[Promptwise Wasserstein ambiguity induces mean-reward ambiguity]
\label{cor:wasserstein-mean}
Fix a prompt $x$ and view each reward function through its vector
\[
R(x)\coloneqq(R(x,y_1),\ldots,R(x,y_n)).
\]
For any $p,q\ge1$, any candidate mean-reward vector $\advreward(x)$, and any budget $\bud\ge0$,
\begin{equation*}
\begin{aligned}
&\exists P\ \text{such that}\ 
\E_{R\sim P}[R(x)]=\advreward(x)
\ \text{and}\ 
\W_{q,d_p}(P,P_0;x)\le\bud\\
&\qquad\Longleftrightarrow\qquad
\left\|\advreward(x)-\E_{R\sim P_0}[R(x)]\right\|_p\le\bud .
\end{aligned}
\end{equation*}
\end{corollary}

\noindent
\Cref{prop:mean-reduction,cor:wasserstein-mean} justify replacing reward-distribution ambiguity by deterministic ambiguity over mean rewards. With $P_0=\delta_{\proxy}$ and transport cost \(d_p\), the induced fixed-prompt set is
\begin{equation}
\label{eq:promptwise-ambiguity}
\mathcal M_{\bud,p}(\proxy(x))\coloneqq
\left\{
\advreward(x)\in\R^n:\ 
\norm{\advreward(x)-\proxy(x)}_p\le \bud
\right\}.
\end{equation}
Thus $\bud$ controls the allowable promptwise reward perturbation: small $\bud$ trusts the proxy, while large $\bud$ asks the policy to hedge against reward-model extrapolation error.

\section{Closed-Form Promptwise Analysis}
\label{sec:promptwise}

We now isolate the core promptwise optimization problem and solve it exactly. The first step is to justify why this reduction is not merely a heuristic convenience. If the policy class is rich enough to realize arbitrary promptwise response distributions, then the global ex-ante DRRO problem factorizes into independent promptwise problems for any fixed norm parameter \(p\).

\begin{proposition}[Promptwise decomposition]
\label{prop:promptwise-decomposition}
Assume the policy class is the full product space $\Pi_{\mathrm{full}}\coloneqq\Delta_n^{\mathcal X}$, suppressing the harmless prompt-dependence of $n$ in the notation. Fix \(1\le p\le\infty\) and define
\[
\mathcal M_{\bud,p}^{\mathcal X}(\proxy)
\coloneqq
\left\{
\advreward:\ \advreward(x)\in\mathcal M_{\bud,p}(\proxy(x))
\text{ for every }x\in\mathcal X
\right\},
\]
and, for $\zeta\in\Delta_n$, write
\[
\Reg_x(\zeta,\advreward(x))
\coloneqq
\max_{\beta\in\Delta_n}\ip{\beta-\zeta}{\advreward(x)}.
\]
Then, for any fixed policy $\pi$, an optimal global adversary is exactly a direct product of optimal promptwise adversaries:
\[
\advreward^\star(x)\in
\argmax_{\advreward(x)\in\mathcal M_{\bud,p}(\proxy(x))}
\Reg_x(\pi(x),\advreward(x))
\quad\text{for }\mathcal D\text{-a.e. }x,
\]
and conversely any such direct product is globally optimal. Moreover, a policy $\pi^\star$ is globally optimal if and only if it is a direct product of optimal promptwise policies:
\[
\pi^\star(x)\in
\argmin_{\zeta\in\Delta_n}
\max_{\advreward(x)\in\mathcal M_{\bud,p}(\proxy(x))}
\Reg_x(\zeta,\advreward(x))
\quad\text{for }\mathcal D\text{-a.e. }x.
\]
\end{proposition}

\begin{remark}[Approximate decomposition for realizable LLM policies]
The proposition is stated for the full product policy class, where the learner may choose a separate distribution for every prompt. A real LLM is not that flexible: the same parameters shape responses across many prompts, and a gradient update must fit all promptwise signals through one shared neural policy. Still, the decomposition is the right local picture when the model class is expressive on the prompts being trained. The promptwise DRRO solution says what each prompt would like the policy to do if it could move freely; the implemented policy-gradient algorithm then asks the LLM to approximate those directions as well as its parameterization allows. In this sense, realizability error enters as a coupling effect across prompts, not as a change in the robust objective itself.
\end{remark}

The proposition says that promptwise analysis is exact for an expressive policy class and therefore a useful structural guide for large language models. The rest of this section fixes a prompt and studies the resulting simplex problem. This is where the main theoretical intuition becomes visible: the adversary has a remarkably explicit dual-norm target, and under the \(\ell_1\) geometry used later by the algorithm, the optimizer becomes a water-filling allocation that protects responses whose proxy rewards are high but whose current policy coverage is too low.

\subsection{Fixed-Prompt Simplex Model}

Fix a prompt $x$ and enumerate its possible complete responses as $y_1,\ldots,y_n$. Write
\begin{equation*}
\pi=(\pi_1,\ldots,\pi_n)\in\Delta_n,\qquad \pi_i=\pi(y_i\mid x),
\end{equation*}
and define the promptwise proxy-reward vector
\begin{equation*}
\proxy=(\proxy_1,\ldots,\proxy_n),\qquad \proxy_i=\proxy(x,y_i).
\end{equation*}
For a norm parameter \(1\le p\le\infty\) and a perturbation $\Delta\in\R^n$ with $\norm{\Delta}_p\le \bud$, let the plausible reward vector be
\begin{equation*}
\advreward=\proxy+\Delta.
\end{equation*}
The promptwise regret of policy $\pi$ under $\advreward$ is
\begin{equation*}
\Reg(\pi,\advreward):=\max_{\beta\in\Delta_n}\ip{\beta-\pi}{\advreward}.
\end{equation*}
Define the fixed-policy worst-case regret
\begin{equation}
\label{eq:psi-p}
\Psi_p(\pi;\proxy)
\coloneqq
\max_{\norm{\Delta}_p\le \bud}\Reg(\pi,\proxy+\Delta).
\end{equation}
The \(\ell_1\) specialization used by the algorithm is the promptwise problem
\begin{equation}
\label{eq:promptwise-drro}
\min_{\pi\in\Delta_n}\ \max_{\norm{\Delta}_1\le \bud}\Reg(\pi,\proxy+\Delta).
\end{equation}

\begin{proposition}[Inner adversary for a fixed policy]
\label{prop:inner-adversary}
Let \(1\le p\le\infty\), and let \(p^\ast\) be the dual exponent, with \(p^\ast=\infty\) when \(p=1\) and \(p^\ast=1\) when \(p=\infty\). For every $\pi\in\Delta_n$,
\begin{equation}
\label{eq:inner-closed-form-lp}
\Psi_p(\pi;\proxy)
=
\max_{1\le i\le n}
\left\{
\proxy_i-\ip{\pi}{\proxy}
+\bud\norm{e_i-\pi}_{p^\ast}
\right\}.
\end{equation}
Let \(k\) be any maximizer in \eqref{eq:inner-closed-form-lp}. If \(1<p<\infty\) and \(e_k\ne\pi\), one optimal perturbation is
\begin{equation}
\label{eq:general-p-adversary}
\Delta_j^\star
=
\bud\,
\frac{\operatorname{sgn}((e_k-\pi)_j)\left|(e_k-\pi)_j\right|^{p^\ast-1}}
{\norm{e_k-\pi}_{p^\ast}^{p^\ast-1}},
\qquad j=1,\ldots,n.
\end{equation}
If \(p=\infty\), one may instead take
\(\Delta_j^\star=\bud\,\operatorname{sgn}((e_k-\pi)_j)\), with any value in \([-\bud,\bud]\) on zero coordinates. Specifically, when \(p=1\), the value reduces to
\begin{equation}
\label{eq:inner-closed-form}
\max_{\norm{\Delta}_1\le \bud}\Reg(\pi,\proxy+\Delta)
=
\bud + \max_{1\le i\le n}(\proxy_i-\bud \pi_i)-\ip{\pi}{\proxy}.
\end{equation}
The inner adversary can then be chosen to place its entire budget on a single response:
\begin{equation*}
\Delta^\star=\bud e_k,\qquad
k\in\argmax_{1\le i\le n}\{\proxy_i-\bud \pi_i\}.
\end{equation*}
\end{proposition}

\noindent
\Cref{prop:inner-adversary} identifies the adversarial target through the dual norm \(\norm{e_i-\pi}_{p^\ast}\). For \(p=1\), this target simplifies to the index $\proxy_i-\bud\pi_i$: responses that have both high proxy reward and insufficient policy mass. That simple index becomes the backbone of the RLHF algorithm developed later.

\begin{proposition}[General \(\ell_p\) epigraph form]
\label{prop:lp-epigraph}
Fix \(1<p\le\infty\), and let \(p^\ast\) be the dual exponent. The minimizers of the promptwise \(\ell_p\)-DRRO problem
\[
\min_{\pi\in\Delta_n}\Psi_p(\pi;\proxy)
\]
are exactly the first components of the solutions to the convex epigraph problem
\begin{equation}
\label{eq:lp-epigraph}
\begin{aligned}
\min_{\pi\in\Delta_n,\ t\in\R}
&\quad t-\ip{\pi}{\proxy}\\
\text{subject to}
&\quad t\ge \proxy_i+\bud\norm{e_i-\pi}_{p^\ast},
\qquad i=1,\ldots,n .
\end{aligned}
\end{equation}
\end{proposition}

\noindent
\Cref{prop:lp-epigraph} gives an exact finite-dimensional convex representation for every \(p>1\). The special case \(p=1\) is sharper: the epigraph constraints become linear after using \(\norm{e_i-\pi}_\infty=1-\pi_i\), and the optimizer can be written in the water-filling form below.

\subsection{Water-Filling Structure of Promptwise Optimizer}

Assume without loss of generality that the responses are labeled so that
\begin{equation}
\label{eq:reward-monotone}
\proxy_1\ge \proxy_2\ge \cdots \ge \proxy_n.
\end{equation}
Define the promptwise robust utility
\begin{equation}
\label{eq:hard-utility}
F_{\bud}(\pi;\proxy):=\ip{\pi}{\proxy}-\max_{1\le i\le n}(\proxy_i-\bud \pi_i),
\end{equation}
which differs from the worst-case regret in \eqref{eq:inner-closed-form} only by the additive constant $\bud$.

Two structural facts already point toward a water-filling solution. First, the robust correction depends on the uncovered rewards $\proxy_i-\bud \pi_i$, so moving mass toward a response is useful only until that uncovered level has been pushed down to the common threshold. Second, once the rewards are ordered as in \eqref{eq:reward-monotone}, there is no reason to hedge a lower-scoring response before the higher-scoring ones have been brought down to the same uncovered level. The optimal policy therefore has to equalize a subset of uncovered rewards and leave any residual mass on the best response.

\begin{theorem}[Water-filling with leftover mass]
\label{thm:water-filling}
Assume \eqref{eq:reward-monotone} and $\bud>0$. Let $t_0$ be the unique solution to
\begin{equation}
\label{eq:t0-equation}
\sum_{i=1}^n (\proxy_i-t)_+ = \bud.
\end{equation}
Define
\begin{equation*}
t^\star := \inf\left\{t\ge t_0:\ \sum_{i:\proxy_i>t}(\proxy_1-\proxy_i)\le \bud\right\}.
\end{equation*}
Then an optimal solution to \eqref{eq:promptwise-drro} is
\begin{equation}
\label{eq:water-filling-solution}
\begin{aligned}
\pi_i^\star
&=\frac{(\proxy_i-t^\star)_+}{\bud},
&& i=2,\ldots,n,\\
\pi_1^\star
&=1-\frac{1}{\bud}\sum_{i=2}^n(\proxy_i-t^\star)_+.
\end{aligned}
\end{equation}
\end{theorem}

\noindent
\Cref{thm:water-filling} is a genuine water-filling rule, but with an RLHF-specific twist. All nonmaximal responses receive precisely the excess of their proxy reward above the optimal threshold $t^\star$, while the remaining probability mass is assigned to the best response. When $\bud$ is small, the threshold is high and the policy stays concentrated. As $\bud$ grows, new responses enter the hedge. Once $\bud$ is large enough that $t^\star=t_0<\proxy_n$, every response receives positive mass and the policy approaches the uniform distribution.
\Cref{app:promptwise-proofs} contains the proof. For the main text, the important point is qualitative: the ambiguity budget lowers a threshold through the ordered proxy rewards, and probability is assigned only to responses whose uncovered rewards remain above that threshold.

\subsection{Interpreting Water-Filling Policy}

The water-filling theorem is useful not only because it is explicit, but also because it explains how the robustness budget changes the policy qualitatively. The policy does not flatten all responses equally. It first hedges against nearby competitors, then gradually widens support as the ambiguity budget grows, and it always keeps any residual probability mass on the top-ranked response.

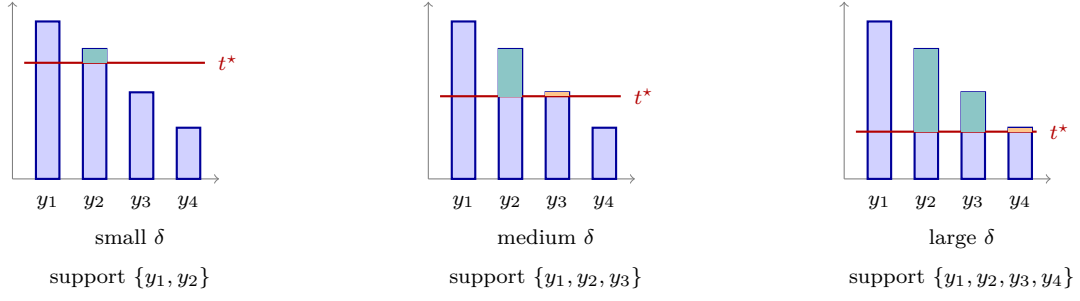
\begin{figure}[t]
\centering
\begin{minipage}[t]{0.31\linewidth}
\centering
\begin{tikzpicture}[x=0.62cm,y=0.52cm,font=\scriptsize]
  \draw[->, gray] (0.2,0) -- (0.2,4.5);
  \draw[->, gray] (0.2,0) -- (4.6,0);
  \foreach \x/\h/\lab in {0.7/4.0/1,1.7/3.3/2,2.7/2.2/3,3.7/1.3/4} {
    \fill[blue!18] (\x,0) rectangle +(0.5,\h);
    \draw[blue!60!black, thick] (\x,0) rectangle +(0.5,\h);
    \node[anchor=north] at (\x+0.25,-0.14) {$y_{\lab}$};
  }
  \draw[red!70!black, thick] (0.45,2.95) -- (4.3,2.95);
  \node[red!70!black, anchor=west] at (4.35,2.95) {$t^\star$};
  \fill[teal!45] (1.7,2.95) rectangle +(0.5,0.35);
\end{tikzpicture}

{\scriptsize small $\bud$\\support $\{y_1,y_2\}$}
\end{minipage}
\hfill
\begin{minipage}[t]{0.31\linewidth}
\centering
\begin{tikzpicture}[x=0.62cm,y=0.52cm,font=\scriptsize]
  \draw[->, gray] (0.2,0) -- (0.2,4.5);
  \draw[->, gray] (0.2,0) -- (4.6,0);
  \foreach \x/\h/\lab in {0.7/4.0/1,1.7/3.3/2,2.7/2.2/3,3.7/1.3/4} {
    \fill[blue!18] (\x,0) rectangle +(0.5,\h);
    \draw[blue!60!black, thick] (\x,0) rectangle +(0.5,\h);
    \node[anchor=north] at (\x+0.25,-0.14) {$y_{\lab}$};
  }
  \draw[red!70!black, thick] (0.45,2.10) -- (4.3,2.10);
  \node[red!70!black, anchor=west] at (4.35,2.10) {$t^\star$};
  \fill[teal!45] (1.7,2.10) rectangle +(0.5,1.20);
  \fill[orange!45] (2.7,2.10) rectangle +(0.5,0.10);
\end{tikzpicture}

{\scriptsize medium $\bud$\\support $\{y_1,y_2,y_3\}$}
\end{minipage}
\hfill
\begin{minipage}[t]{0.31\linewidth}
\centering
\begin{tikzpicture}[x=0.62cm,y=0.52cm,font=\scriptsize]
  \draw[->, gray] (0.2,0) -- (0.2,4.5);
  \draw[->, gray] (0.2,0) -- (4.6,0);
  \foreach \x/\h/\lab in {0.7/4.0/1,1.7/3.3/2,2.7/2.2/3,3.7/1.3/4} {
    \fill[blue!18] (\x,0) rectangle +(0.5,\h);
    \draw[blue!60!black, thick] (\x,0) rectangle +(0.5,\h);
    \node[anchor=north] at (\x+0.25,-0.14) {$y_{\lab}$};
  }
  \draw[red!70!black, thick] (0.45,1.20) -- (4.3,1.20);
  \node[red!70!black, anchor=west] at (4.35,1.20) {$t^\star$};
  \fill[teal!45] (1.7,1.20) rectangle +(0.5,2.10);
  \fill[teal!45] (2.7,1.20) rectangle +(0.5,1.00);
  \fill[orange!45] (3.7,1.20) rectangle +(0.5,0.10);
\end{tikzpicture}

{\scriptsize large $\bud$\\support $\{y_1,y_2,y_3,y_4\}$}
\end{minipage}
\caption{A water-filling view of \Cref{thm:water-filling}. The colored blocks show the hedge assigned to nonmaximal responses above the threshold $t^\star$; the remaining mass is sent to the best response $y_1$. As the ambiguity budget grows, the threshold drops and additional responses enter the active support.}
\label{fig:water-filling}
\end{figure}

\noindent
\Cref{fig:water-filling} makes the mechanics of the optimizer visible. Small ambiguity budgets trust the proxy and therefore keep most probability on the best response, with only the nearest competitor receiving hedge mass. As $\bud$ increases, the threshold $t^\star$ falls and further responses enter the active support one by one. What matters is not their absolute rank alone, but whether their uncovered rewards still sit above the current water level.

\paragraph{Sanity checks.}
The limiting regimes line up with the intended interpretation of $\bud$. As $\bud\downarrow 0$, the cost of hedging dominates and the optimizer converges to the proxy-greedy policy $e_1$. At the other extreme, once $\bud>\sum_{i=1}^n(\proxy_i-\proxy_n)$ every response is active, so $t^\star=t_0=\bar{\proxy}-\bud/n$ with $\bar{\proxy}=\frac{1}{n}\sum_{i=1}^n \proxy_i$, and the water-filling rule simplifies to
\begin{equation*}
\pi_i^\star=\frac{1}{n}+\frac{\proxy_i-\bar{\proxy}}{\bud},\qquad i=1,\ldots,n.
\end{equation*}
Hence $\pi^\star$ approaches the uniform distribution as $\bud\to\infty$. DRRO therefore interpolates continuously between proxy-greedy optimization and complete hedging, which is exactly the behavior one would want from an ambiguity budget.

\section{Algorithm for RLHF}
\label{sec:algorithm}

The exact promptwise solution from \Cref{sec:promptwise} tells us what a regret-robust policy should look like once a prompt has been reduced to a finite simplex. RLHF training, however, never gets to move probability mass on that simplex directly. What it actually sees are sampled completions from a language model, and what it actually updates are model parameters.

Throughout this section we specialize to the \(\ell_1\) ambiguity geometry, whose locality is what makes the sampled bonus implementable; \Cref{sec:algorithm-discussion,prop:l1-rationale} explains this choice.

Those two viewpoints are close enough to guide the algorithm, but different enough to create real work. For a fixed prompt, the response space is astronomically large, so the adversarial coordinate from \Cref{prop:inner-adversary} is hidden from us and can only be glimpsed through a small sampled group. And even if the promptwise optimizer were available in closed form, it would still not produce a usable LLM update, because the policy is an LLM parameterized by $\theta$ and the trainable object is $\theta$, not the full vector $\pi(\cdot\mid x)$. The way forward is therefore to rewrite the robust objective as a policy-gradient recipe, understand what bonus the exact theory implies, and then relax that bonus until it can be estimated reliably from grouped rollouts.

\subsection{From Promptwise Simplex to Policy Gradients}

The theory above treats a complete response as one action. This is the right abstraction for the robust objective: the reward model scores completed responses, GRPO groups completed responses, and the policy-gradient update uses the sequence log-probability supplied by the language model. Thus we can work at the completion level without losing compatibility with standard LLM post-training.

For a fixed prompt $x$, enumerate the candidate complete responses as $y_1,\ldots,y_n$ and write
\begin{equation*}
\pi_{\theta,i}=\pi_\theta(y_i\mid x),
\qquad
\proxy_i=\proxy(x,y_i).
\end{equation*}
The nominal promptwise RLHF target is simply the expected proxy reward
\begin{equation*}
\ip{\pi_\theta(\cdot\mid x)}{\proxy(x)}=\sum_{i=1}^n \pi_{\theta,i}\proxy_i.
\end{equation*}
Differentiating this nominal target with respect to $\theta$ gives
\begin{equation*}
\nabla_\theta \ip{\pi_\theta(\cdot\mid x)}{\proxy(x)}
=
\sum_{i=1}^n \proxy_i \nabla_\theta \pi_\theta(y_i\mid x)
=
\sum_{i=1}^n \pi_\theta(y_i\mid x) \proxy_i \nabla_\theta\log\pi_\theta(y_i\mid x).
\end{equation*}
Equivalently,
\begin{equation*}
\nabla_\theta \ip{\pi_\theta(\cdot\mid x)}{\proxy(x)}
=
\E_{y\sim\pi_\theta(\cdot\mid x)}
\big[
\proxy(x,y)\nabla_\theta\log\pi_\theta(y\mid x)
\big].
\end{equation*}
Hence, if $y^{(1)},\ldots,y^{(K)}\sim \pi_\theta(\cdot\mid x)$ are sampled on policy, then
\begin{equation*}
\widehat g_\theta^{\mathrm{nom}}(x)
:=
\frac{1}{K}\sum_{k=1}^K
\proxy(x,y^{(k)})\nabla_\theta\log\pi_\theta(y^{(k)}\mid x)
\end{equation*}
is an unbiased Monte Carlo estimator of the nominal policy gradient above. When the samples instead come from a frozen rollout policy $\rollpi=\pi_{\theta_{\mathrm{old}}}$, standard importance sampling yields the detached estimator
\begin{equation*}
\widehat g_{\theta,\rollpi}^{\mathrm{nom}}(x)
:=
\frac{1}{K}\sum_{k=1}^K
\rho_\theta(x,y^{(k)})\,\proxy(x,y^{(k)})\nabla_\theta\log\pi_\theta(y^{(k)}\mid x),
\qquad
y^{(k)}\sim\rollpi(\cdot\mid x),
\end{equation*}
where $\rho_\theta(x,y)=\pi_\theta(y\mid x)/\rollpi(y\mid x)$.

Now fix a rollout policy $\rollpi=\pi_{\theta_{\mathrm{old}}}$ and let $\widetilde r_{\rollpi}(x,y)$ be any promptwise reward signal computed from that frozen policy. This setup is deliberately broad: standard RLHF corresponds to $\widetilde r_{\rollpi}=\proxy$, while our DRRO variants will use shaped rewards derived from the robust theory. Define
\begin{equation*}
\widetilde J_{\rollpi}(\theta)
:=
\E_{x\sim\mathcal D,\ y\sim\pi_\theta(\cdot\mid x)}[\widetilde r_{\rollpi}(x,y)].
\end{equation*}
Because the reward is detached during the policy update, the score-function identity gives
\begin{equation*}
\nabla_\theta \widetilde J_{\rollpi}(\theta)
=
\E_{x\sim\mathcal D,\ y\sim\pi_\theta(\cdot\mid x)}
\big[
\widetilde r_{\rollpi}(x,y)\nabla_\theta\log\pi_\theta(y\mid x)
\big].
\end{equation*}
Equivalently, if the responses were generated from $\rollpi$, then
\begin{equation*}
\widetilde J_{\rollpi}(\theta)
=
\E_{x\sim\mathcal D,\ y\sim\rollpi(\cdot\mid x)}
\Big[
\rho_\theta(x,y)\,\widetilde r_{\rollpi}(x,y)
\Big],
\qquad
\rho_\theta(x,y):=\frac{\pi_\theta(y\mid x)}{\rollpi(y\mid x)}.
\end{equation*}
This is the starting point for GRPO-style clipped surrogates: one samples groups of responses from a frozen rollout policy, computes detached rewards or advantages on those responses, and then performs several stable gradient steps on $\theta$. The two algorithm boxes in this section should therefore be read as policy-gradient procedures with an explicit outer rollout/update loop.

\subsection{GRPO as Standard Grouped RLHF Backbone}

Before inserting any robust correction, it helps to recall the grouped RLHF template that practitioners already use. The importance-sampling estimator above is the original recipe: estimate the gradient of the nominal promptwise target $\ip{\pi_\theta(\cdot\mid x)}{\proxy(x)}$ from sampled completions and then ascend along that estimate. Following \citet{shao2024deepseekmath}, GRPO stabilizes that recipe by grouping completions promptwise, replacing raw rewards by normalized within-group advantages, and clipping the importance ratio inside the policy update. Concretely, for each prompt $x$ we sample a group of completions
\begin{equation*}
y^{(1)},\ldots,y^{(G)} \sim \rollpi(\cdot\mid x),
\end{equation*}
evaluate them with the proxy reward model,
\begin{equation*}
\widehat r^{(k)}=\proxy(x,y^{(k)}),\qquad k=1,\ldots,G,
\end{equation*}
and normalize the groupwise rewards into relative advantages
\begin{equation}
\label{eq:grpo-adv}
\begin{aligned}
\bar r
&:=\frac{1}{G}\sum_{\ell=1}^G \widehat r^{(\ell)},\qquad
s_r^2:=\frac{1}{G}\sum_{\ell=1}^G(\widehat r^{(\ell)}-\bar r)^2,\\
A_{\mathrm{GRPO}}^{(k)}
&:=
\frac{\widehat r^{(k)}-\bar r}{s_r+\varepsilon_{\mathrm{adv}}}.
\end{aligned}
\end{equation}
The completion-level GRPO surrogate is then
\begin{equation}
\label{eq:grpo-loss}
\begin{aligned}
\mathcal L_{\mathrm{GRPO}}(\theta;\rollpi)
&:=
\E\!\left[
\frac{1}{G}\sum_{k=1}^G
\ell_\theta^{(k)}
\right],\\
\ell_\theta^{(k)}
&:=
\min\!\left\{
\rho_\theta^{(k)}A_{\mathrm{GRPO}}^{(k)},
c_\theta^{(k)}A_{\mathrm{GRPO}}^{(k)}
\right\}.
\end{aligned}
\end{equation}
\noindent
where the expectation is over $x\sim\mathcal D$ and $y^{(1)},\ldots,y^{(G)}\sim\rollpi(\cdot\mid x)$, and
\begin{equation*}
\begin{aligned}
\rho_\theta^{(k)}
&:=
\frac{\pi_\theta(y^{(k)}\mid x)}
{\rollpi(y^{(k)}\mid x)},\\
c_\theta^{(k)}
&:=
\operatorname{clip}
\bigl(\rho_\theta^{(k)},
1-\varepsilon_{\mathrm{clip}},
1+\varepsilon_{\mathrm{clip}}\bigr).
\end{aligned}
\end{equation*}
This is the baseline recipe to which our DRRO corrections will be compared. Everything that follows keeps the grouped rollout, importance-ratio, and clipped policy-gradient backbone of \eqref{eq:grpo-loss}; only the reward signal inside the group changes. In particular, we do not add any explicit KL penalty term to the policy objective.

\begin{algorithm}[t]
\caption{Standard GRPO-style RLHF}
\label{alg:grpo}
\begin{algorithmic}[1]
\Require initial policy $\pi_\theta$, prompt source $\mathcal D$, reward model $\proxy$, outer iterations $M$, prompt batch size $B$, group size $G$, clip radius $\varepsilon_{\mathrm{clip}}$, policy-gradient steps $S$
\Ensure updated policy $\pi_\theta$
\For{outer iteration $m=1,\ldots,M$}
  \State set rollout policy $\rollpi\gets\pi_\theta$
  \State sample prompt batch $x_1,\ldots,x_B\sim\mathcal D$
  \For{each prompt $x_b$ in the batch}
  \State sample $y_b^{(1)},\ldots,y_b^{(G)}\sim\rollpi(\cdot\mid x_b)$
  \State compute proxy rewards $\widehat r_b^{(k)}=\proxy(x_b,y_b^{(k)})$
  \State compute grouped advantages $A_{b,\mathrm{GRPO}}^{(k)}$ from \eqref{eq:grpo-adv}
  \State compute response ratios $\rho_b^{(k)}(\theta)=\pi_\theta(y_b^{(k)}\mid x_b)/\rollpi(y_b^{(k)}\mid x_b)$
  \State accumulate the clipped GRPO surrogate \eqref{eq:grpo-loss}
  \EndFor
  \State update $\theta$ using $S$ minibatch policy-gradient steps on \eqref{eq:grpo-loss}
\EndFor
\end{algorithmic}
\end{algorithm}

\subsection{Soft-Max Relaxation and Sampled Estimation}

The exact promptwise DRRO objective contains the maximum
\(\max_i(\proxy_i-\bud\pi_i)\). On the finite simplex this term is perfectly well defined, but in RLHF it is not directly observable: the response space is too large to enumerate, and a rollout group shows only a small random subset of possible completions. Estimating the maximum by taking the largest value inside the sampled group is therefore not faithful to the target in general, because the response attaining the true maximum over the hidden simplex may simply not have been sampled.

We therefore replace the exact maximum by a smooth log-sum-exp relaxation. The point of this relaxation is not cosmetic: it converts the adversarial correction into an expectation under a soft adversarial distribution, which can be approximated from sampled completions. For a fixed prompt, define
\begin{equation}
\label{eq:soft-objective}
\begin{aligned}
F_{\bud,\tau}(\pi;\proxy)
&:=
\ip{\pi}{\proxy}
-\tau\log Z_{\tau}(\pi;\proxy),\\
Z_{\tau}(\pi;\proxy)
&:=
\sum_{i=1}^n
\exp\!\left(\frac{\proxy_i-\bud \pi_i}{\tau}\right),
\qquad \tau>0.
\end{aligned}
\end{equation}
This is the smooth target we will optimize through sampled completions.

\begin{proposition}[Soft-max approximation error]
\label{prop:soft-approx}
For every promptwise policy $\pi\in\Delta_n$ and proxy-reward vector $\proxy\in\R^n$,
\begin{equation}
\label{eq:soft-approx}
F_{\bud,\tau}(\pi;\proxy)
\le
F_{\bud}(\pi;\proxy)
\le
F_{\bud,\tau}(\pi;\proxy)+\tau\log n.
\end{equation}
\end{proposition}

\noindent
Thus smaller temperature means a smaller approximation gap to the exact promptwise objective. The price is that the induced adversarial distribution becomes more concentrated, so approximating that distribution from samples becomes harder. This is the central tradeoff in the soft formulation: low temperature reduces optimization bias, but it also makes the sampled approximation of the extra regret term less forgiving.

The smooth objective has an exact gradient that again takes the form of reward shaping.

\begin{proposition}[Gradient of the soft DRRO objective]
\label{prop:soft-gradient}
Fix a prompt $x$ and treat $\bud$ as fixed during the policy-gradient step. Define
\begin{equation*}
\pi_{\theta,i}:=\pi_\theta(y_i\mid x),\qquad
\proxy_i:=\proxy(x,y_i),\qquad
\sigma_{\theta,x}(y_i):=
\frac{\exp((\proxy_i-\bud\pi_{\theta,i})/\tau)}
{\sum_{j=1}^n \exp((\proxy_j-\bud\pi_{\theta,j})/\tau)}.
\end{equation*}
Then
\begin{equation}
\label{eq:soft-gradient}
\begin{aligned}
g_{\theta}^{\mathrm{soft}}(x)
&:=
\nabla_\theta F_{\bud,\tau}
(\pi_\theta(\cdot\mid x);\proxy(x)),\\
g_{\theta}^{\mathrm{soft}}(x)
&=
\sum_{i=1}^n \pi_\theta(y_i\mid x)
R_i^{\mathrm{soft}}
\nabla_\theta\log\pi_\theta(y_i\mid x),\\
R_i^{\mathrm{soft}}
&:=
\proxy_i+\textcolor{blue!70!black}{B_i^{\mathrm{soft}}},\\
B_i^{\mathrm{soft}}
&:=
\bud\,\sigma_{\theta,x}(y_i).
\end{aligned}
\end{equation}
\end{proposition}

\noindent
Equation~\eqref{eq:soft-gradient} has the same policy-gradient form as nominal RLHF, except that the proxy reward is augmented by the robust bonus \(B_i^{\mathrm{soft}}\). This bonus is proportional to the soft adversarial probability mass \(\sigma_{\theta,x}(y_i)\). That representation is exactly what makes the correction estimable from partial samples: instead of identifying a single maximizer on the hidden simplex, we only need to approximate expectations under \(\sigma_{\theta,x}\).

To estimate these expectations, we use self-normalized importance sampling (SNIS). Suppose there is a proposal distribution $q(\cdot\mid x)$ over possible completions whose support covers the responses receiving nonnegligible soft-adversarial mass. For a fixed prompt and fixed budget $\bud$, define
\begin{equation*}
a_{\theta,x}(y)
:=
\exp\!\left(\frac{\proxy(x,y)-\bud\,\pi_\theta(y\mid x)}{\tau}\right),
\qquad
u_q(y):=\frac{a_{\theta,x}(y)}{q(y\mid x)}.
\end{equation*}
For samples $y^{(1)},\ldots,y^{(K)}\sim q(\cdot\mid x)$, set
\begin{equation}
\label{eq:snis-weights}
u^{(k)}:=u_q(y^{(k)}),
\qquad
w^{(k)}:=\frac{u^{(k)}}{\sum_{\ell=1}^K u^{(\ell)}}.
\end{equation}
Then, for any bounded test function $h$, the corresponding exact soft expectation can be written as
\begin{equation*}
\mu_h(x)
:=
\sum_{i=1}^n \sigma_{\theta,x}(y_i)h(y_i)
=
\frac{\E_{y\sim q(\cdot\mid x)}[u_q(y)h(y)]}
{\E_{y\sim q(\cdot\mid x)}[u_q(y)]},
\end{equation*}
and the SNIS approximation is
\begin{equation}
\label{eq:snis-estimator}
\widehat\mu_{h,K}(x):=\sum_{k=1}^K w^{(k)} h(y^{(k)}).
\end{equation}

The choice of proposal distribution is itself a substantive sampling question. In principle, $q$ could be any language-model proposal with adequate support. In this paper we use the rollout policy $q=\rollpi$, for a practical reason: the proposal distribution must itself be available during generation and log-probability computation, while memory is already tight in contemporary LLM inference. The rollout policy is already resident in the GRPO loop, so choosing $q=\rollpi$ avoids storing or serving another proposal model while keeping the sampled estimator aligned with the policy update.

Finite-sample bias and concentration results for SNIS are classical; see, for example, \citet{cardoso2022brsnis}. For our grouped RLHF estimator, the following promptwise bound makes the dependence on the sample size $K$ explicit. In the statement, write $u(y)=u_q(y)$. In the implementation below, one simply takes $K=G$.

\begin{proposition}[Finite-sample SNIS error]
\label{prop:snis-finite}
Fix a prompt $x$ and a bounded test function $h$. Suppose $0\le u(y)\le U_x$ and $|h(y)|\le H_x$ for $q(\cdot\mid x)$-almost every $y$, and write
\begin{equation*}
\nu_x:=\E_{y\sim q(\cdot\mid x)}[u(y)]>0.
\end{equation*}
Then for every $\eta\in(0,1)$, if
\begin{equation*}
K\ge \frac{2U_x^2}{\nu_x^2}\log\frac{4}{\eta},
\end{equation*}
the estimator \eqref{eq:snis-estimator} satisfies
\begin{equation}
\label{eq:snis-finite-bound}
\begin{aligned}
\Pr(E_K)&\ge 1-\eta,\\
E_K&:=
\left\{
\left|\widehat\mu_{h,K}(x)-\mu_h(x)\right|
\le B_K
\right\},\\
B_K&:=
\frac{4U_xH_x}{\nu_x}
\sqrt{\frac{\log(4/\eta)}{2K}}.
\end{aligned}
\end{equation}
\end{proposition}

\noindent
The dependence on the temperature is now visible in two places. By \Cref{prop:soft-approx}, lower $\tau$ makes the optimization target closer to the exact promptwise DRRO objective. But in \eqref{eq:snis-finite-bound}, the constant $U_x$ grows with the dynamic range of $\exp((\proxy-\bud \pi)/\tau)$, so low temperature also increases the variance and decreases the effective sample size of the SNIS weights. The temperature therefore trades approximation bias against sampling stability.

For the rollout choice $q=\rollpi$, we evaluate the soft adversarial weights at the frozen rollout policy and write
\[
\widehat r^{(k)}:=\proxy(x,y^{(k)}),
\qquad
\bar\pi^{(k)}:=\rollpi(y^{(k)}\mid x).
\]
The soft reward-shaping gradient implied by \eqref{eq:soft-gradient} is then approximated by
\begin{equation}
\label{eq:on-policy-estimator}
\widehat g_\theta(x)
=
\frac{1}{G}\sum_{k=1}^G
\widetilde r_{\mathrm{raw}}^{(k)}
\nabla_\theta\log\pi_\theta(y^{(k)}\mid x).
\end{equation}
Here the SNIS weights are computed from \eqref{eq:snis-weights} with $q=\rollpi$ and $\pi_\theta$ in $a_{\theta,x}$ frozen at $\rollpi$. Equivalently, one may define the soft DRRO-shaped rewards
\begin{equation}
\label{eq:soft-bonus-raw}
\widetilde r_{\mathrm{raw}}^{(k)}
=
\widehat r^{(k)}+\textcolor{blue!70!black}{\bud\,G w^{(k)}\bar\pi^{(k)}}.
\end{equation}
Here the first term is again the ordinary proxy reward, while the second term is the robust bonus. The factor $G w^{(k)}\bar\pi^{(k)}$ is the SNIS approximation to the soft adversarial mass assigned to response $y^{(k)}$.

The raw expression above is the conceptual soft-SNIS estimator. The practical implementation uses the group-normalized stabilization and scaled budget introduced in \Cref{sec:delta-selection}; this changes the numerical coordinate system of the sampled group, not the underlying soft DRRO target. In other words, we are now optimizing a grouped policy-gradient surrogate for the smooth promptwise target
\begin{equation*}
J_{\bud,\tau}^{\mathrm{soft}}(\pi)
:=
\E_{x\sim\mathcal D}
\big[
F_{\bud,\tau}(\pi(\cdot\mid x);\proxy(x,\cdot))
\big].
\end{equation*}

\subsection{\texorpdfstring{Selection of Ambiguity $\delta$ and Final Implemented Algorithm}{Selection of Ambiguity delta and Final Implemented Algorithm}}
\label{sec:delta-selection}

At this point the mechanics are almost complete. The remaining design choice is how much ambiguity to grant the adversary at different stages of training.

\paragraph{Numerical stabilization.}
Before selecting the budget, we first translate the raw sampled soft bonus in \eqref{eq:soft-bonus-raw} into the coordinate system used by a finite rollout group. Raw completion probabilities $\bar\pi^{(k)}$ can be extremely small, so the implementation uses group-normalized probabilities. Within each sampled group we set
\begin{equation*}
\widetilde\pi^{(k)}
\coloneqq
\frac{\bar\pi^{(k)}}{\sum_{\ell=1}^G\bar\pi^{(\ell)}}.
\end{equation*}
Under the heuristic scale $\bar\pi^{(k)}\approx 1/n$, this normalization magnifies probabilities by about $n/G$. At the same time, the original Wasserstein budget $\bud$ scales with the dimension $n$ of the conceptual reward vector, so the sampled implementation uses the scaled budget $\widetilde{\bud}\coloneqq(G/n)\bud$. The practical soft DRRO bonus is
\begin{equation}
\label{eq:practical-soft-bonus}
\begin{aligned}
\widetilde u^{(k)}
&\coloneqq
\frac{1}{\widetilde\pi^{(k)}}
\exp\!\left(
\frac{\widehat r^{(k)}
-\widetilde{\bud}\,\widetilde\pi^{(k)}}{\tau}
\right),\\
\widetilde w^{(k)}
&\coloneqq
\frac{\widetilde u^{(k)}}
{\sum_{\ell=1}^G\widetilde u^{(\ell)}},\\
\widetilde r_{\mathrm{DRRO}}^{(k)}
&\coloneqq
\widehat r^{(k)}
+\textcolor{blue!70!black}{
\widetilde{\bud}\,G\,\widetilde w^{(k)}
\widetilde\pi^{(k)}}.
\end{aligned}
\end{equation}
This is an implementation-level stabilization of the same soft DRRO target, not a different robust objective: the full-simplex theory still supplies the soft adversarial distribution, while the sampled update rescales the tiny rollout probabilities into a numerically usable group-level coordinate system. We then plug the group-normalized versions of $\widetilde r_{\mathrm{DRRO}}^{(k)}$ into the same clipped GRPO surrogate \eqref{eq:grpo-loss}.

\paragraph{Dynamic ambiguity budget.}
Keeping $\bud$ fixed throughout training is unnecessarily rigid because reward-model error typically grows as the policy drifts away from the reference model. The Donsker--Varadhan variational formula gives a clean way to formalize that idea.

\begin{proposition}[Donsker--Varadhan control of reward misspecification]
\label{prop:dv-budget}
Fix a prompt $x$ and define
\begin{equation*}
h_x(y):=\big|\proxy(x,y)-\true(x,y)\big|.
\end{equation*}
Then for every $\lambda>0$,
\begin{equation}
\label{eq:dv-bound}
\E_{y\sim\pi_\theta(\cdot\mid x)}[h_x(y)]
\le
\frac{1}{\lambda}
\left[
\KL\big(\pi_\theta(\cdot\mid x)\,\|\,\pi_0(\cdot\mid x)\big)
+
\log\E_{y\sim\pi_0(\cdot\mid x)}\!\big[e^{\lambda h_x(y)}\big]
\right].
\end{equation}
\end{proposition}

\noindent
The formal inequality still contains the unknown log-moment term under $\pi_0$, but its three terms explain why KL drift is a sensible observable proxy for reward uncertainty. The left-hand side is an analogue of the promptwise $\ell_1$ error between the proxy reward and the certainty-equivalent true reward; the log-moment term can be treated as a baseline constant because it is evaluated under the fixed reference policy rather than the current policy; and the KL term is the observable policy-drift statistic that remains, with the tunable coefficient $\alpha$ in \eqref{eq:dynamic-budget} playing the role of the multiplier $1/\lambda$. In the implementation we therefore use the scaled sampled budget
\begin{equation}
\label{eq:dynamic-budget}
\widetilde{\bud}_\theta(x)=\widetilde{\bud}_0+\alpha\,\KL\big(\rollpi(\cdot\mid x)\,\|\,\pi_0(\cdot\mid x)\big),
\end{equation}
where $\widetilde{\bud}_0\ge 0$ is a pilot-calibrated base budget and $\alpha>0$ is a scaling coefficient. This KL quantity is used only to calibrate the ambiguity budget; it is not added as an explicit penalty to the policy objective.

In experiments, the KL term is estimated on sampled completions with the well-known $k_3$ estimator of \citet{schulman2020kl}: for $z^{(k)}=\log\pi_0(y^{(k)}\mid x)-\log\rollpi(y^{(k)}\mid x)$, average $\exp(z^{(k)})-z^{(k)}-1$ over the group. This statistic is nonnegative samplewise, has lower variance than the raw log-ratio estimator, and is differentiated through neither the budget nor the policy loss.

The calibration is deliberately pragmatic. We first draw a small pilot set of prompts, generate a modest number of responses from the initial policy, and score those responses with both the proxy reward model and a stronger held-out gold reward model. The empirical promptwise $\ell_1$ discrepancy between the two provides an initial guess for the unscaled budget and hence for $\widetilde{\bud}_0$ after group normalization. After that, the remaining hyperparameters---notably $\alpha$, $\tau$, the group size $G$, and the clipping radius $\varepsilon_{\mathrm{clip}}$---are tuned on a standard held-out validation set evaluated by the stronger reward model.

The final implementation uses the soft bonus, the dynamic budget, and the unchanged GRPO backbone. Once the scaled promptwise budgets $\widetilde{\bud}_b$ have been instantiated from \eqref{eq:dynamic-budget}, the only code-level modification inside the GRPO update is still the reward-shaping block.

\begin{algorithm}[t]
\caption{Final DRRO-RLHF implementation (blue lines differ from \Cref{alg:grpo})}
\label{alg:drro-rlhf}
\begin{algorithmic}[1]
\Require initial policy $\pi_\theta$, prompt source $\mathcal D$, reward model $\proxy$, reference policy $\pi_0$, scaled base budget $\widetilde{\bud}_0$, scaling coefficient $\alpha$, outer iterations $M$, prompt batch size $B$, group size $G$, temperature $\tau$, clip radius $\varepsilon_{\mathrm{clip}}$, policy-gradient steps $S$
\Ensure updated policy $\pi_\theta$
\For{outer iteration $m=1,\ldots,M$}
  \State set rollout policy $\rollpi\gets\pi_\theta$
  \State sample prompt batch $x_1,\ldots,x_B\sim\mathcal D$
  \For{each prompt $x_b$ in the batch}
  \State sample $y_b^{(1)},\ldots,y_b^{(G)}\sim\rollpi(\cdot\mid x_b)$
  \State compute proxy rewards $\widehat r_b^{(k)}=\proxy(x_b,y_b^{(k)})$
  \State instantiate scaled budget $\widetilde{\bud}_b=\widetilde{\bud}_0+\alpha\,\KL(\rollpi(\cdot\mid x_b)\,\|\,\pi_0(\cdot\mid x_b))$
  \State \algchange{compute rollout probabilities $\bar\pi_b^{(k)}=\rollpi(y_b^{(k)}\mid x_b)$ and normalized probabilities $\widetilde\pi_b^{(k)}$}
  \State \algchange{set shaped rewards $\widetilde r_{b,\mathrm{DRRO}}^{(k)}$ according to \eqref{eq:practical-soft-bonus}}
  \State normalize \algchange{$\widetilde r_{b,\mathrm{DRRO}}^{(k)}$} within the prompt group to form advantages
  \State compute response ratios $\rho_b^{(k)}(\theta)=\pi_\theta(y_b^{(k)}\mid x_b)/\rollpi(y_b^{(k)}\mid x_b)$
  \State accumulate the clipped GRPO surrogate \eqref{eq:grpo-loss}
  \EndFor
  \State update $\theta$ using $S$ minibatch policy-gradient steps on \eqref{eq:grpo-loss}
\EndFor
\end{algorithmic}
\end{algorithm}

\subsection{Discussion}
\label{sec:algorithm-discussion}

Two practical points deserve a final discussion: how much extra computation DRRO introduces on top of GRPO, and why the $\ell_1$ ambiguity set is the geometry that makes that extra computation tractable. Both points reinforce the same message of this section, namely that the robust correction should be viewed as a light reward-shaping layer rather than as a new RLHF training stack.

\paragraph{Computational overhead.}
The additional burden is small. During autoregressive rollout, the model already returns the tokenwise conditional probabilities needed to form each completion probability $\bar\pi^{(k)}=\rollpi(y^{(k)}\mid x)$, so no extra model query is required; one only accumulates the emitted token probabilities along the sampled response. Once the proxy rewards are available, soft DRRO adds $G$ scalar exponentials, one normalization of the sampled rollout probabilities, and the corresponding reshaped rewards. Relative to sequence generation, reward-model scoring, and the usual minibatch policy-gradient updates, these operations are negligible. This is the operational reason why the DRRO modification fits naturally into mature GRPO codebases.

\paragraph{Choice of ambiguity geometry.}
The choice of the $\ell_1$ ambiguity set is equally practical. The promptwise adversary has a dual-norm form for general $\ell_p$ geometry, but $\ell_1$ is the case that turns this structure into a water-filling policy and a sampled RLHF bonus. In the algorithm, its key advantage is locality: the robust correction can be written in terms of the probability assigned to the threatening sampled response rather than the entire hidden simplex. The next proposition makes this statement precise by comparing $\ell_p$ ambiguity geometries through the bonus factor attached to a vertex response.

\begin{proposition}[Endpoint norms and coordinate-local sampled bonuses]
\label{prop:l1-rationale}
Fix $n\ge3$. Suppose the promptwise ambiguity set is $\norm{\Delta}_p\le \bud$ and let $p^\ast$ be the dual exponent satisfying $1/p+1/p^\ast=1$. Then
\begin{equation}
\label{eq:lp-dual-rewrite}
\begin{aligned}
\Psi_p(\pi;\proxy)
&:=
\max_{\norm{\Delta}_p\le \bud}
\Reg(\pi,\proxy+\Delta),\\
\Psi_p(\pi;\proxy)
&=
\max_i
\left\{
\proxy_i-\ip{\pi}{\proxy}
+\bud b_{p,i}(\pi)
\right\},\\
b_{p,i}(\pi)
&\coloneqq\norm{e_i-\pi}_{p^\ast} .
\end{aligned}
\end{equation}
The vertex bonus factor $b_{p,i}(\pi)$ depends only on the sampled-coordinate probability $\pi_i$ for all $\pi\in\Delta_n$ if and only if $p\in\{1,\infty\}$. In those endpoint cases,
\begin{equation*}
b_{1,i}(\pi)=1-\pi_i,
\qquad
b_{\infty,i}(\pi)=2(1-\pi_i).
\end{equation*}
Consequently, $\ell_\infty$ ambiguity induces the same promptwise policy dependence as $\ell_1$ ambiguity after doubling the budget:
\begin{equation*}
\max_{\norm{\Delta}_\infty\le\bud}\Reg(\pi,\proxy+\Delta)
=
2\bud+\max_i\{\proxy_i-2\bud\pi_i\}-\ip{\pi}{\proxy}.
\end{equation*}
\end{proposition}

\noindent
The proposition has a useful interpretation. For any nonendpoint norm $1<p<\infty$, a sampled response cannot be assigned a robust bonus from its own rollout probability alone: the bonus also depends on how the policy distributes probability across all other, mostly unobserved responses. The two endpoint geometries are exceptional. The case $p=\infty$ is coordinate-local too, but it is essentially the same policy correction as $\ell_1$ after multiplying the budget by two, while its adversarial perturbation is dense rather than sparse. We therefore use $\ell_1$: it gives the local sampled bonus needed for implementation, keeps the worst-case reward perturbation concentrated on the threatening response, and gives the budget the direct meaning of total promptwise reward misspecification mass. Once the response space is only partially observed through grouped rollouts, this sparse and local structure is exactly what makes the DRRO correction estimable.

\section{Experiments}
\label{sec:experiments}

This section reports the empirical study of DRRO for RLHF. The experiments are designed to isolate the decision-theoretic question studied in the paper: when a policy is optimized against an imperfect proxy reward, does regret robustness preserve useful improvement better than value robustness and other mitigation strategies?

\paragraph{Setup.}
All methods use the same VERL-based RLHF stack \citep{sheng2024hybridflow} with Ray workers \citep{moritz2018ray} and vLLM rollout \citep{kwon2023efficient}. We use \path{HuggingFaceH4/hh-rlhf} \citep{bai2022training} as prompts, Qwen2.5-0.5B-Instruct \citep{qwen2024qwen25} as the policy, a fine-tuned OpenAssistant DeBERTa-v3-base reward model \citep{he2021debertav3,openassistant2023rewardbase} as the proxy reward, and the Tasksource DeBERTa-v3-large RLHF reward model \citep{sileod2023tasksource} as the held-out evaluator. Performance is reward improvement over the initial policy: for evaluator $r_{\mathrm{eval}}$, we report $\E[r_{\mathrm{eval}}(x,Y_\pi)]-\E[r_{\mathrm{eval}}(x,Y_{\pi_0})]$ on held-out prompts. ``Proxy'' is the training reward model; ``gold'' is the separate held-out reward model. The gold evaluator is never used for policy updates. \Cref{app:experimental-details} gives fine-tuning, proxy-quality, implementation, and hyperparameter details.

The main benchmark sits at intermediate proxy quality: the fine-tuned proxy has 85.7\% pairwise agreement with the held-out evaluator, between the raw-proxy stress test (73.6\%) and large-proxy sanity check (95.2\%) in \Cref{app:experimental-details}. All reported main curves average five seeds, with shaded regions showing one standard deviation across seeds. Rewards are reported as improvements relative to the step-0 initial policy; proxy and gold improvements are each re-centered by their own step-0 baseline. The horizontal axis is sequence-level KL from the frozen initial policy, estimated on sampled completions from the current policy; the exact estimator is given in \Cref{app:experimental-details}. This KL is logged only for diagnostics and plotting and is not an explicit penalty in the training loss.

\paragraph{Benchmark methods.}
The comparison isolates value-versus-regret robustness while including common over-optimization mitigations. PPO is the standard clipped actor-critic baseline \citep{schulman2017ppo}; GRPO removes the critic and uses grouped within-prompt reward normalization \citep{shao2024deepseekmath}. Ensemble-Mean and Ensemble-UWO follow reward-model ensemble mitigation \citep{coste2024ensembles,eisenstein2024helping}: the former averages reward models, while the latter subtracts ensemble disagreement. BSPO discourages movement into out-of-distribution responses \citep{dai2025behavior}, and InfoRM represents reward-proxy improvement through information-theoretic reward modeling \citep{miao2024inform}. DRO-RLHF implements the standard value-robust benchmark \citep{ben2013robust,wiesemann2014distributionally,esfahani2018data} with the same rollout and budget interface as DRRO. DRRO-RLHF is evaluated in a hard sampled-bonus variant, which gives the robust bonus to the sampled completion with largest $\widehat r^{(k)}-\widetilde\delta\,\widetilde\pi^{(k)}$, and in the soft dynamic-budget variant from \Cref{sec:algorithm}.

For the robust methods, DRO and DRRO receive the same budget information; only the robust benchmark changes. The default fixed scaled budget is $\widetilde\delta=2.5G$, and the displayed soft-dynamic DRRO setting uses dynamic coefficient $\alpha=10$ and soft-assignment temperature $\tau=2.0$. Dynamic-budget runs use \eqref{eq:dynamic-budget}.

\begin{figure}[t]
\centering
\includegraphics[width=0.82\linewidth]{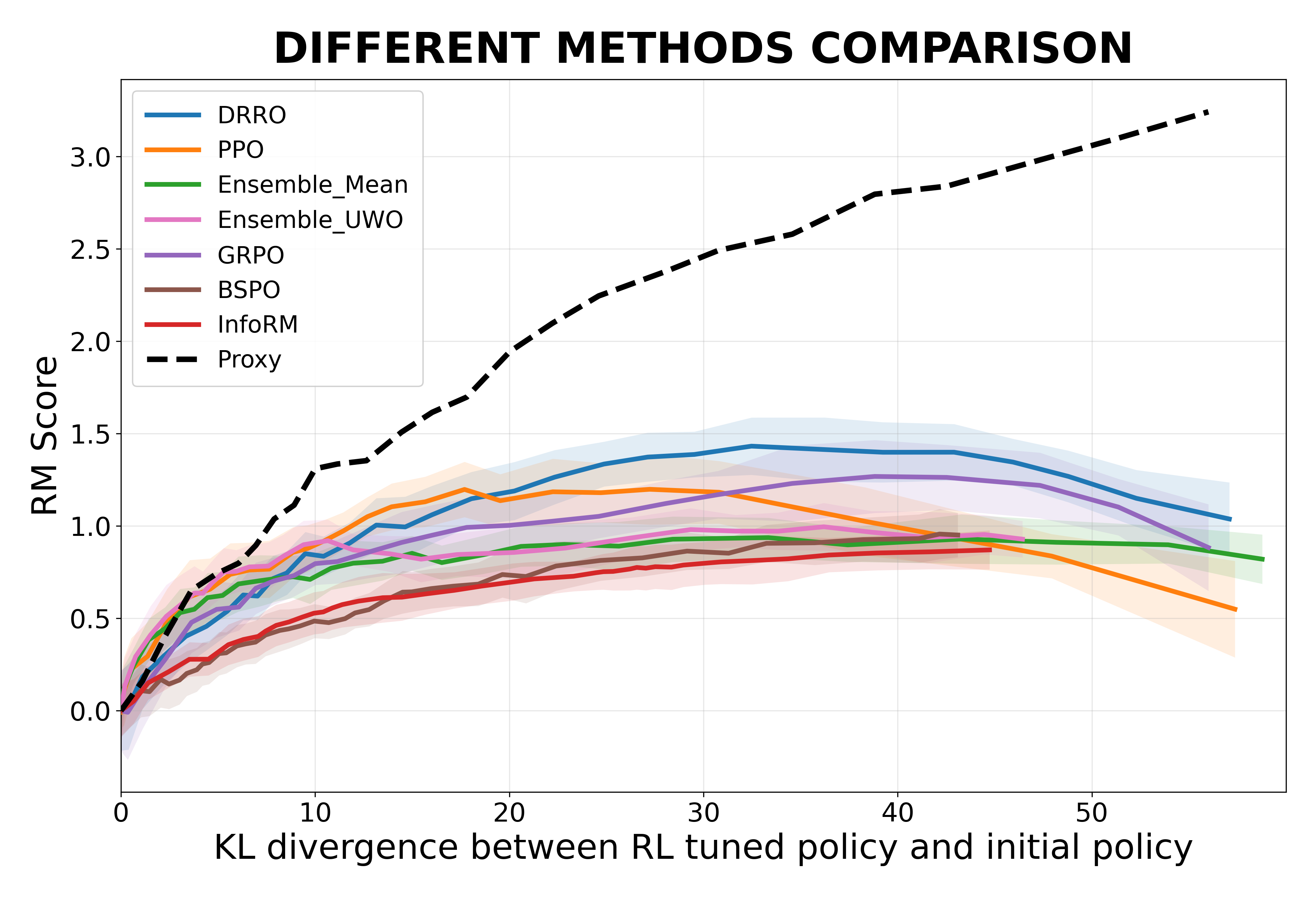}
\caption{Main benchmark across RLHF update rules and mitigation baselines. Rewards are improvements over the initial policy, with proxy and gold rewards each re-centered by their own step-0 baseline. Solid curves use the held-out gold reward model. The dashed proxy curve uses the DRRO proxy scores. Curves are averaged over five seeds, and shaded regions show mean $\pm$ one standard deviation.}
\label{fig:main-benchmark}
\end{figure}

\begin{table}[t]
\centering
\caption{Peak held-out gold reward and diagnostics. The proxy column is measured at the same KL as peak gold.}
\label{tab:main-results}
\small
\setlength{\tabcolsep}{4.2pt}
\begin{tabular}{lrrrr}
\toprule
Method & Peak gold & Proxy at peak & Gold-proxy gap & Peak KL \\
\midrule
PPO & 1.20 & 2.35 & -1.15 & 27.23 \\
GRPO & 1.27 & 2.80 & -1.53 & 38.82 \\
Ensemble-Mean & 0.94 & 2.55 & -1.61 & 33.34 \\
Ensemble-UWO & 1.00 & 2.66 & -1.67 & 36.19 \\
BSPO & 0.96 & 2.83 & -1.88 & 42.15 \\
InfoRM & 0.87 & 2.90 & -2.03 & 44.73 \\
DRO-RLHF & 0.79 & 2.55 & -1.76 & 33.26 \\
DRRO-RLHF (hard) & 1.17 & 2.42 & -1.25 & 29.09 \\
DRRO-RLHF (soft + dynamic) & \textbf{1.43} & 2.53 & -1.10 & 32.47 \\
\bottomrule
\end{tabular}
\end{table}

\paragraph{Main comparison.}
\Cref{fig:main-benchmark,tab:main-results} show the central pattern. The figure gives reward-versus-KL trajectories, and the table reports each method's best held-out gold point with its proxy value, gold-proxy gap, and KL from the initial policy. PPO and GRPO make early progress, but proxy and gold separate as KL grows; DRRO attains the largest peak gold reward while maintaining a smaller gold-proxy gap than the strongest proxy-chasing baselines.

Numerically, soft dynamic DRRO reaches peak gold reward $1.43$, above GRPO ($1.27$), PPO ($1.20$), and hard DRRO ($1.17$), at moderate KL $32.47$ and with the smallest strong-method gold-proxy gap ($-1.10$). GRPO, BSPO, and InfoRM can have larger proxy-at-peak values but larger negative gaps, so DRRO's improvement is not proxy chasing. DRO has similar KL to DRRO ($33.26$ versus $32.47$) but only $0.79$ peak gold, illustrating value-robust over-pessimism. Thus \Cref{fig:main-benchmark} should be read as a frontier plot, while \Cref{fig:dro-vs-drro} isolates the value-robust baseline.

\begin{figure}[t]
\centering
\begin{subfigure}[t]{0.48\linewidth}
\centering
\includegraphics[width=\linewidth]{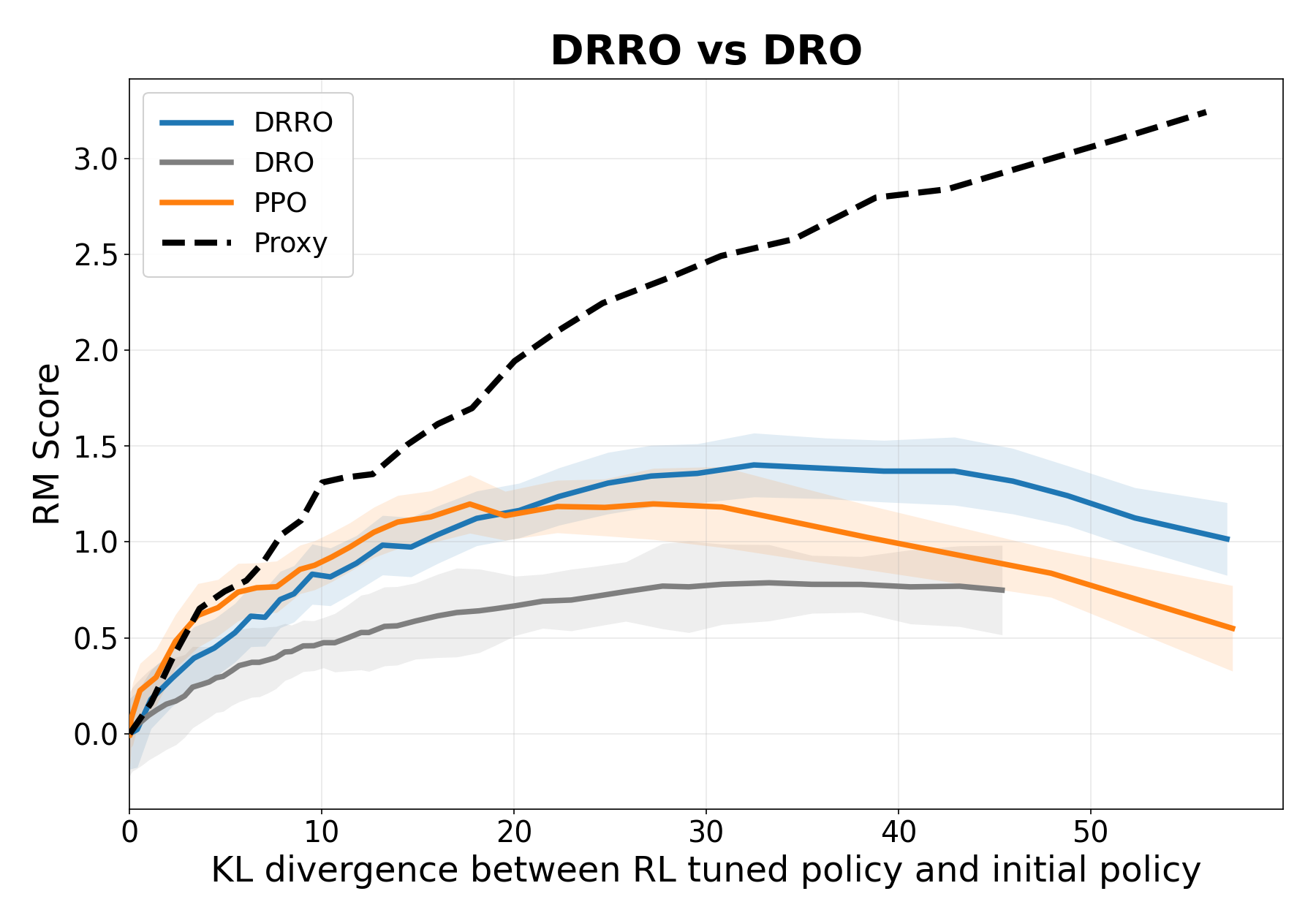}
\caption{Matched DRO/DRRO comparison.}
\label{fig:dro-vs-drro}
\end{subfigure}
\hfill
\begin{subfigure}[t]{0.48\linewidth}
\centering
\includegraphics[width=\linewidth]{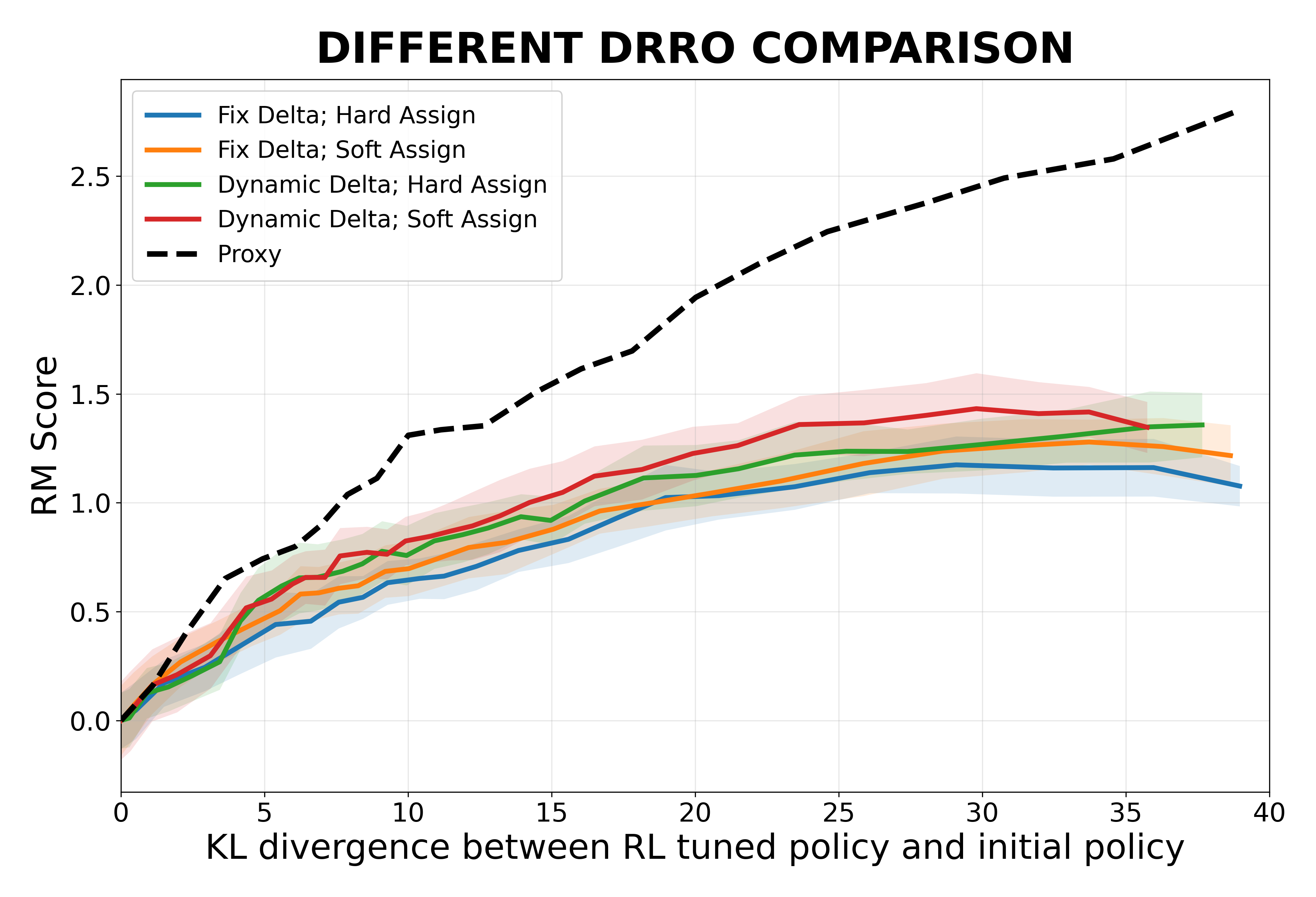}
\caption{DRRO ablation.}
\label{fig:ablation}
\end{subfigure}
\caption{Left: DRRO improves the reward-conservatism tradeoff over DRO. Right: dynamic budgeting and soft assignment both improve DRRO.}
\label{fig:extra-experiments}
\end{figure}

\paragraph{DRO and ablations.}
\Cref{fig:dro-vs-drro} isolates the value-versus-regret benchmark. DRO and DRRO use the same rollout stack and budget information, but DRO pessimizes absolute value while DRRO pessimizes regret. This empirical gap anticipates the theoretical comparison in \Cref{sec:coverage}: value robustness can be too conservative because it protects a policy in isolation, whereas regret robustness protects the comparison that actually drives the decision. \Cref{fig:ablation} shows that both DRRO design choices matter: soft assignment avoids committing to one noisy completion, while dynamic budgeting grows ambiguity with KL drift. The four DRRO variants in \Cref{fig:ablation} use hyperparameters selected by the grid search described in \Cref{app:experimental-details}; the hard-max ablation algorithm is detailed in \Cref{app:hard-max-ablation}.

\section{Comparing DRO and DRRO}
\label{sec:coverage}

The experiments make the value-versus-regret distinction visible before the theory names it. DRO protects absolute value and can therefore become conservative even when the proxy still contains useful comparative information. DRRO protects regret, so its pessimism is relative to the policies that compete under the same plausible reward perturbation. This section formalizes that difference through two lenses: rank preservation and local coverage.

\subsection{Rank-Preserving Proxy Rewards}

Pairwise reward-model training is fundamentally ordinal. Bradley--Terry-style objectives and modern RLHF reward-model pipelines learn scores by assigning larger values to preferred responses than to rejected ones \citep{bradley1952rank,christiano2017deep,stiennon2020learning,ouyang2022training}. This is why held-out pairwise accuracy and ranking correlation remain standard diagnostics \citep{lambert2025rewardbench}, even though accuracy alone does not determine downstream RLHF performance \citep{chen2024accuracy}. A natural test is therefore whether the robust objective can exploit a proxy ranking that is correct but not perfectly calibrated.

For the fixed-prompt $\ell_1$ specialization used by the water-filling rule and the algorithm, promptwise DRO solves
\begin{equation*}
\max_{\pi\in\Delta_n}\ \ip{\pi}{\proxy} - \bud\norm{\pi}_\infty.
\end{equation*}
The penalty is blind to which responses the proxy ranks highest. DRRO keeps that ordering inside the regret term. The next theorem shows that, under rank preservation, this difference translates into higher true value.

\begin{theorem}[DRRO dominates DRO under rank-preserving truth]
\label{thm:drro-dominates-dro}
Assume the proxy reward is strictly ordered,
\begin{equation*}
\proxy_1>\proxy_2>\cdots>\proxy_n,
\end{equation*}
and that the true reward is a rank-preserving transform of the proxy reward,
\begin{equation*}
t_i=\phi(\proxy_i),\qquad i=1,\ldots,n,
\end{equation*}
for some increasing function $\phi$. Let
\begin{equation*}
\pi^{\mathrm{DRO}}(\bud)\in\argmax_{\pi\in\Delta_n}\left\{\ip{\pi}{\proxy}-\bud\norm{\pi}_\infty\right\},
\end{equation*}
and
\begin{equation*}
\pi^{\mathrm{DRRO}}(\bud)\in\argmax_{\pi\in\Delta_n}\left\{\ip{\pi}{\proxy}-\max_i(\proxy_i-\bud \pi_i)\right\}.
\end{equation*}
Then
\begin{equation*}
\sum_{i=1}^n \pi_i^{\mathrm{DRRO}}(\bud)\,t_i
\ge
\sum_{i=1}^n \pi_i^{\mathrm{DRO}}(\bud)\,t_i.
\end{equation*}
If $\phi$ is strictly increasing and the two optimizers differ, then the displayed inequality is strict.
\end{theorem}

\noindent
The proof in \Cref{app:drro-vs-dro} shows that the DRRO optimizer majorizes the DRO optimizer. DRO equalizes exposure because it protects worst-case value; DRRO reallocates mass only when doing so reduces regret against plausible competitors. When the proxy ranking is correct, DRRO therefore preserves more mass on the best-ranked responses.

\subsection{Coverage as an Alternative View}

The rank-preserving theorem isolates a promptwise mechanism. A complementary view asks what the robust guarantees pay for after prompts and reward uncertainty are aggregated. Because \Cref{fig:main-benchmark} shows that all methods can deteriorate after drifting too far from the initial policy, we compare DRO and DRRO inside a local policy class.

Let $\mathrm d(\pi,\pi_0)$ be a policy-drift statistic, with the expected KL divergence to the initial policy as the canonical RLHF example, and define the local policy class
\begin{equation*}
\Pi_\rho
:=
\{\pi\in\Pi:\mathrm d(\pi,\pi_0)\le \rho\}.
\end{equation*}
This captures the trust-region viewpoint without adding a KL penalty to the objective. For each prompt $x$, let $z_\pi(x)$ be a linear policy representation: in the finite-response case, $z_\pi(x)=\pi(\cdot\mid x)$; in sequential settings, it can be an occupancy measure. Then
\begin{equation*}
J_r(\pi)=\E_{x\sim\mathcal D}\big[\ip{z_\pi(x)}{r(x)}\big].
\end{equation*}
Let $\mu_x$ be a full-support coverage distribution, such as the behavior distribution induced by reward-model training data. Following the concentrability logic common in offline RL \citep{munos2008finite,zhan2022offline}, and keeping the geometry aligned with the promptwise $\ell_1$ model, define
\begin{equation*}
\norm{e}_{1,\mu_x}:=\sum_{y\in\mathcal Y}\mu_x(y)|e(y)|,
\qquad
\norm{v}_{\infty,\mu_x^{-1}}:=\max_{y\in\mathcal Y}\frac{|v(y)|}{\mu_x(y)},
\end{equation*}
where $\mu_x(y)>0$ for every response $y$. Assume the mean weighted $\ell_1$ confidence set
\begin{equation*}
\U^{(1)}_\varepsilon:=
\left\{
r:\ \E_{x\sim\mathcal D}\big[\norm{r(x)-\proxy(x)}_{1,\mu_x}\big]\le \varepsilon
\right\}.
\end{equation*}
This integrated condition is implied by the pointwise promptwise $\ell_1$ ambiguity condition in the algorithmic specialization, since $\mu_x$ is a probability distribution.

Define
\begin{equation*}
\begin{aligned}
\pi_\rho^{\mathrm{DRO}}
&\in
\argmax_{\pi\in\Pi_\rho}\ \min_{r\in\U^{(1)}_\varepsilon}J_r(\pi),\\
\pi_\rho^{\mathrm{DRRO}}
&\in
\argmin_{\pi\in\Pi_\rho}\ \max_{r\in\U^{(1)}_\varepsilon}\Big[\max_{\beta\in\Pi_\rho}J_r(\beta)-J_r(\pi)\Big].
\end{aligned}
\end{equation*}
Let $\pi_\rho^\star\in\argmax_{\pi\in\Pi_\rho}J_{\true}(\pi)$ denote the best local policy under the true reward.

For DRO, the relevant quantity is absolute exposure.

\begin{definition}[Absolute integrated concentrability]
\label{def:coverage}
For a policy $\pi\in\Pi_\rho$, define
\begin{equation*}
C_\infty(\pi;\mu,\mathcal D)
:=
\operatorname*{ess\,sup}_{x\sim\mathcal D}
\norm{z_\pi(x)}_{\infty,\mu_x^{-1}}.
\end{equation*}
\end{definition}

\noindent
This is the product-space weighted $\ell_\infty$ dual of the mean $\ell_1$ reward-error set. It yields the following absolute-value certificate.

\begin{proposition}[DRO certificate]
\label{prop:dro-certificate}
For every $\pi\in\Pi_\rho$,
\begin{equation}
\label{eq:dro-certificate}
\min_{r\in\U^{(1)}_\varepsilon}J_r(\pi)
=
J_{\proxy}(\pi)-\varepsilon C_\infty(\pi;\mu,\mathcal D).
\end{equation}
\end{proposition}

\noindent
The same certificate yields a local regret guarantee.

\begin{proposition}[DRO regret bound]
\label{prop:dro-regret}
Suppose $\true\in\U^{(1)}_\varepsilon$. Then
\begin{equation}
\label{eq:dro-regret-bound}
J_{\true}(\pi_\rho^\star)-J_{\true}(\pi_\rho^{\mathrm{DRO}})
\le
2\varepsilon C_\infty(\pi_\rho^\star;\mu,\mathcal D).
\end{equation}
\end{proposition}

For DRRO, uncertainty enters through policy differences.

\begin{definition}[Relative integrated concentrability]
\label{def:relative-concentrability}
For policies $\pi,\beta\in\Pi_\rho$, define
\begin{equation*}
C_{\infty,\mathrm{rel}}(\beta,\pi;\mu,\mathcal D)
:=
\operatorname*{ess\,sup}_{x\sim\mathcal D}
\norm{z_\beta(x)-z_\pi(x)}_{\infty,\mu_x^{-1}}.
\end{equation*}
\end{definition}

\noindent
This relative coefficient is the object that appears in regret.

\begin{proposition}[DRRO certificate]
\label{prop:drro-certificate}
For every $\pi\in\Pi_\rho$,
\begin{equation}
\label{eq:drro-certificate}
\begin{aligned}
&\max_{r\in\U^{(1)}_\varepsilon}
\Big[\max_{\beta\in\Pi_\rho}J_r(\beta)-J_r(\pi)\Big]\\
&\qquad =
\max_{\beta\in\Pi_\rho}
\left\{
J_{\proxy}(\beta)-J_{\proxy}(\pi)
+\varepsilon C_{\infty,\mathrm{rel}}(\beta,\pi;\mu,\mathcal D)
\right\}.
\end{aligned}
\end{equation}
\end{proposition}

\noindent
The guarantee depends only on comparators that can be optimal for some reward in the confidence set. Define
\begin{equation*}
\Pi_\rho^\star(\U^{(1)}_\varepsilon):=\bigcup_{r\in\U^{(1)}_\varepsilon}\argmax_{\pi\in\Pi_\rho}J_r(\pi),
\end{equation*}
and the localized relative-concentrability radius at a candidate policy $\pi$
\begin{equation*}
\mathcal C_{\infty,\mathrm{rel}}(\U^{(1)}_\varepsilon,\pi;\mu,\mathcal D)
:=
\max_{\beta\in\Pi_\rho^\star(\U^{(1)}_\varepsilon)}
C_{\infty,\mathrm{rel}}(\beta,\pi;\mu,\mathcal D).
\end{equation*}

\begin{proposition}[DRRO regret bound with localized relative concentrability]
\label{prop:drro-regret}
Suppose $\true\in\U^{(1)}_\varepsilon$. Then
\begin{equation*}
J_{\true}(\pi_\rho^\star)-J_{\true}(\pi_\rho^{\mathrm{DRRO}})
\le
2\varepsilon\mathcal C_{\infty,\mathrm{rel}}(\U^{(1)}_\varepsilon,\pi_\rho^\star;\mu,\mathcal D).
\end{equation*}
\end{proposition}

\begin{remark}[Comparison between DRO and DRRO regret bounds]
\label{rem:less-pessimistic}
DRO pays for absolute concentrability, while DRRO pays for relative concentrability between the chosen policy and plausible local comparators. For a fixed prompt, the relative term is
\[
\max_{y\in\mathcal Y}\frac{|z_\beta(x,y)-z_\pi(x,y)|}{\mu_x(y)},
\]
which can cancel common exposure shared by nearby policies. The reverse triangle inequality and the triangle inequality give
\begin{equation*}
C_\infty(\pi;\mu,\mathcal D)-C_\infty(\beta;\mu,\mathcal D)
\le
C_{\infty,\mathrm{rel}}(\beta,\pi;\mu,\mathcal D)
\le
C_\infty(\pi;\mu,\mathcal D)+C_\infty(\beta;\mu,\mathcal D).
\end{equation*}
Writing $\mathcal B_\varepsilon\coloneqq\Pi_\rho^\star(\U^{(1)}_\varepsilon)$, $C(\gamma)\coloneqq C_\infty(\gamma;\mu,\mathcal D)$, and $\mathfrak C(\pi)\coloneqq\mathcal C_{\infty,\mathrm{rel}}(\U^{(1)}_\varepsilon,\pi;\mu,\mathcal D)$, maximizing over $\beta\in\mathcal B_\varepsilon$ gives
\begin{equation*}
C(\pi)-\min_{\beta\in\mathcal B_\varepsilon}C(\beta)
\le
\mathfrak C(\pi)
\le
C(\pi)+\max_{\beta\in\mathcal B_\varepsilon}C(\beta).
\end{equation*}
For a fixed confidence set, the comparator-side extrema are constants. Thus DRO behaves like optimizing with an absolute policy-exposure penalty, while DRRO uses the largest relative exposure to plausible comparators. The absolute-exposure view is a triangle-inequality envelope of the relative one, explaining why DRO regret guarantees can be looser.
\end{remark}

\section{Conclusion}
\label{sec:conclusion}

RLHF is optimization under a misspecified reward: the policy is trained against a learned proxy, but deployment quality is judged by latent human preferences. The key modeling choice is therefore where to place pessimism. We argue that regret is the right target. Worst-case value protects a policy in isolation; worst-case regret compares it with the hindsight-best policy under the same plausible reward, matching the loss induced by objective misspecification more directly.

Under norm-induced Wasserstein ambiguity, the promptwise RLHF problem admits a dual-norm closed-form adversary. In the $\ell_1$ specialization, this structure further yields a water-filling optimal policy and a coordinate-local bonus that can be translated into a GRPO-compatible reward-shaping algorithm with light computational overhead. Empirically, the resulting DRRO correction mitigates over-optimization more effectively than the tested baselines, while standard value-robust DRO is systematically over-pessimistic.

\section{Code and Data Disclosure}
\label{sec:code-data}

An anonymized implementation is available at \url{https://anonymous.4open.science/status/DRRO_anonymous-1E42}. It contains the DRRO-GRPO training code, requirements, launch scripts, and reproduction instructions.

The experiments use open-source assets: prompts from Hugging Face H4 HH-RLHF, derived from Anthropic HH-RLHF \citep{bai2022training}; Qwen2.5-0.5B-Instruct as the initial policy \citep{qwen2024qwen25,qwen25modelcard}; OpenAssistant DeBERTa-v3 reward models as proxies \citep{he2021debertav3,openassistant2023rewardbase,openassistant2023rewardlarge}; and the Tasksource DeBERTa-v3-large RLHF reward model as the held-out evaluator \citep{sileod2023tasksource}. We do not redistribute modified datasets or model checkpoints; URLs and implementation details are provided in the repository and \Cref{app:experimental-details}.

\bibliographystyle{plainnat}
\bibliography{references}

\newpage
\appendix
\renewcommand{\theHsection}{appendix.\arabic{section}}
\renewcommand{\theHsubsection}{appendix.\arabic{section}.\arabic{subsection}}

\section{Notation Summary}
\label{app:notation}

\begin{center}
\captionsetup{type=table,hypcap=false,skip=3pt}
\footnotesize
\renewcommand{\arraystretch}{0.92}
\caption{Notation used throughout the paper.}
\label{tab:notation}
\begin{tabular}{p{0.24\linewidth}p{0.67\linewidth}}
\toprule
Notation & Meaning \\
\midrule
$x\sim\mathcal D$ & prompt sampled from the training or deployment distribution \\
$y\in\mathcal Y$ & complete response to the current prompt $x$ \\
$y=(y_{(1)},\ldots,y_{(|y|)})$ & token sequence forming one complete response \\
$y_1,\ldots,y_n$ & conceptual enumeration of all complete responses for a fixed prompt \\
$y^{(1)},\ldots,y^{(G)}$ & sampled completions in one rollout group \\
$\pi_\theta(y\mid x)$ & policy over complete responses \\
$\pi_0,\rollpi$ & frozen reference policy and rollout policy used for one policy-gradient update \\
$\widehat r^{(k)},\bar\pi^{(k)},\widetilde\pi^{(k)}$ & proxy reward, rollout probability, and group-normalized rollout probability for sampled completion $y^{(k)}$ \\
$\proxy(x,y)$ & learned proxy reward model \\
$\true_u(x,y)$ & latent human-preference reward for user $u$ \\
$\true(x,y)$ & certainty-equivalent true population reward $\bar r_{\mathcal P^\star}(x,y)$ \\
$\mathcal P^\star$ & population law over latent reward functions \\
$J_P(\pi),J_r(\pi)$ & expected return of policy $\pi$ under reward law $P$ or deterministic reward $r$ \\
$\Pi$ & policy class under consideration \\
$\Delta_n$ & probability simplex over $n$ enumerated promptwise responses \\
$\proxy(x)$ & promptwise proxy-reward vector after enumerating responses \\
$\advreward(x)$ & promptwise plausible mean-reward vector \\
$d_p,\W_{q,d_p}$ & promptwise $\ell_p$ ground metric and induced $q$-Wasserstein distance over reward laws \\
$\mathcal M_{\bud,p}(\proxy(x))$ & promptwise mean-reward ambiguity set induced by the $d_p$ transport cost \\
$p^\ast$ & dual norm exponent satisfying $1/p+1/p^\ast=1$, with the usual endpoint conventions \\
$t=(t_1,\ldots,t_n)$ & promptwise true reward vector used in the DRRO-vs-DRO comparison theorem \\
$\bud$ & ambiguity-budget radius \\
$\Reg_x(\pi,\advreward)$ & promptwise regret of $\pi$ relative to the best policy under reward $\advreward(x)$ \\
$\mu_x$ & promptwise coverage distribution induced by reward-model data or behavior support \\
$\Pi_\rho$ & local policy class within drift radius $\rho$ of the initial policy $\pi_0$ \\
$C_\infty(\pi;\mu,\mathcal D)$ & absolute product-space $\ell_\infty$ concentrability coefficient \\
$C_{\infty,\mathrm{rel}}(\beta,\pi;\mu,\mathcal D)$ & relative product-space $\ell_\infty$ concentrability coefficient \\
\bottomrule
\end{tabular}
\end{center}

\section{Literature Review}
\label{sec:litreview}

\paragraph{RLHF and objective misspecification.}
RLHF emerged from the preference-learning framework of \citet{christiano2017deep}, who used human comparisons to train reward functions for reinforcement learning agents. \citet{ziegler2019fine} and \citet{stiennon2020learning} adapted that logic to language modeling and summarization, while \citet{ouyang2022training} turned it into the now-standard LLM pipeline of supervised fine-tuning, reward modeling, and reference-anchored policy optimization, often implemented with KL control. Direct preference methods such as DPO \citep{rafailov2023dpo} eliminate the explicit reward-model optimization stage, but they do not eliminate the underlying problem of learning from a finite proxy for human preference. Among recent surveys, \citet{kaufmann2025survey} provide a broad map of RLHF across domains, \citet{chaudhari2024deciphered} focus on reward-model misspecification in LLM alignment, and \citet{liu2026statistical} emphasize preference learning, uncertainty quantification, and statistical inference.

\paragraph{Over-optimization and mitigation strategies.}
The empirical evidence is now fairly systematic. \citet{gao2023scaling} document the canonical proxy-versus-gold divergence in RLHF; \citet{rafailov2024directoveropt} show that the same phenomenon persists in direct alignment algorithms; and \citet{hosking2024gold} argue that human-feedback scores themselves can be misaligned with true quality because they over-reward stylistic features such as assertiveness. On the mitigation side, one family of papers tries to improve the reward model itself. \citet{coste2024ensembles} and \citet{eisenstein2024helping} study reward-model ensembles as a hedge against reward hacking; \citet{miao2024inform} use an information-theoretic reward-model objective; \citet{yang2024hidden} regularize hidden states to improve generalization; \citet{liu2025rrm} target artifact robustness; and \citet{bukharin2025adversarial} generate adversarial hard examples for reward-model training. A related line updates the reward model as the policy moves, including the iterative preference-learning method of \citet{xiong2024iterative}, the active preference-optimization framework of \citet{das2024active}, and the off-policy correction of \citet{ackermann2025offpolicy}.

A second family leaves the reward model largely intact but modifies the policy objective. \citet{zhai2024uprlhf} penalize reward uncertainty estimated from LoRA ensembles, \citet{xu2025pet} learn a pessimistic reward model, and \citet{yan2024rewardrobust} optimize a reward-robust policy objective. A third family regularizes the policy toward regions where the reward model is thought to be reliable. \citet{moskovitz2024constrained} impose explicit constraints in RLHF, \citet{dai2025behavior} regularize toward behavior-supported regions, \citet{rita2024demonstration} use demonstration-guided reinforcement learning, and \citet{liu2024provably} show that SFT loss can play the role of an adversarial regularizer. These methods are important points of comparison for us because they reveal the central tension: protecting against reward misspecification is valuable, but methods built around pessimistic value or hard feasible-region restrictions can also suppress useful policy improvement. \citet{xu2025pet} make this concern explicit by arguing that KL-based regularization can exclude high-performing policies with large KL divergence, and \citet{li2025geb} show that divergence-based bonuses can reinforce conservative behavior rather than encourage exploration of uncertain but promising regions.

\paragraph{DRO and DRRO in operations research.}
Standard DRO protects decisions against ambiguity in probabilities, distributions, and objective coefficients. \citet{ben2013robust} study robust solutions under uncertain probabilities, \citet{wiesemann2014distributionally} provide a general convex-optimization framework, \citet{esfahani2018data} develop data-driven Wasserstein DRO, \citet{blanchet2019quantifying} quantify distributional model risk, and \citet{rahimian2019review} survey the area. The concern that worst-case value may be overly conservative leads naturally to regret-based robustness. \citet{averbakh2004minmax} and \citet{averbakh2005complexity} develop classical minmax-regret models and complexity results; \citet{chen2021newsvendor} study regret in the newsvendor problem; and \citet{agarwal2022minimax}, \citet{cho2024wasserstein}, \citet{bitar2024distributionally}, and \citet{fiechtner2025wasserstein} develop distributionally robust regret formulations in modern statistical and learning settings. Our paper imports that line of work into RLHF and shows that the regret benchmark becomes especially natural once reward uncertainty, rather than transition uncertainty, is the main source of misspecification.

\section{Discussion on KL Penalty Term}
\label{app:kl-discussion}

The KL term should nevertheless be interpreted carefully, because it remains a common reference point in RLHF. For a suitable multiplier $\eta$, the penalized formulation is the Lagrangian form of the trust-region problem
\begin{equation*}
\max_{\pi\in\Pi}\ J_{\proxy}(\pi)
\qquad\text{subject to}\qquad
\E_{x\sim\mathcal D}\!\left[\KL\!\big(\pi(\cdot\mid x)\,\|\,\pi_0(\cdot\mid x)\big)\right]\le \varepsilon.
\end{equation*}
Thus KL control is one way to mitigate over-optimization by shrinking the feasible policy set around the reference model. That safeguard can be useful, but it is not intrinsic to RLHF and it is no longer the only practical recipe. Recent large-scale LLM-RL implementations already train successfully without an explicit KL penalty: \citet{hu2025openreasonerzero} show that vanilla PPO with rule-based rewards can scale reasoning performance without any KL regularization, while \citet{yu2025dapo} explicitly remove the KL term from GRPO-style long-chain-of-thought training and argue that large divergence from the initial model need not be suppressed when it supports better reasoning.

More importantly, KL control is not exactly to the point for the misspecification problem studied here. If the raw proxy reward is imperfect, then maximizing that same proxy over a smaller KL ball is still a constrained optimization against a misspecified objective. In that sense, KL regularization is a last-resort hedge: it limits how much damage the proxy can do by limiting how far the policy may move, but it does not repair the proxy itself. Our viewpoint is that a more direct mitigation is to improve the reward signal being optimized. Once a robust correction is added to the proxy reward, optimizing the corrected surrogate over a given feasible region is more natural than preserving the uncorrected proxy and relying on the feasible-region restriction alone.

\section{Mean-Reward Reduction and Promptwise Ambiguity}
\label{app:model-proofs}

\begin{proof}[Proof of \Cref{prop:mean-reduction}]
By Fubini's theorem,
\begin{align*}
J_P(\pi)
&=
\E_{R\sim P}\E_{x\sim\mathcal D,\ y\sim\pi(\cdot\mid x)}[R(x,y)]\\
&=
\E_{x\sim\mathcal D,\ y\sim\pi(\cdot\mid x)}\E_{R\sim P}[R(x,y)]\\
&=
\E_{x\sim\mathcal D,\ y\sim\pi(\cdot\mid x)}[\bar r_P(x,y)]\\
&=
J_{\bar r_P}(\pi).
\end{align*}
Thus the ex-ante DRRO objective sees a distribution over reward functions only through its mean reward.
\end{proof}

\begin{proof}[Proof of \Cref{cor:wasserstein-mean}]
First suppose that $P$ satisfies
\[
\E_{R\sim P}[R(x)]=\advreward(x),
\qquad
\W_{q,d_p}(P,P_0;x)\le \bud .
\]
For any coupling $\gamma\in\Gamma(P,P_0)$, write
$\bar r_0(x)\coloneqq\E_{\widetilde R\sim P_0}[\widetilde R(x)]$
and $D_x\coloneqq R(x)-\widetilde R(x)$. Then
\begin{equation*}
\begin{aligned}
\norm{\advreward(x)-\bar r_0(x)}_p
&=
\norm{\E_\gamma[D_x]}_p\\
&\le
\E_\gamma[\norm{D_x}_p]\\
&\le
\left(\E_\gamma[\norm{D_x}_p^q]\right)^{1/q},
\end{aligned}
\end{equation*}
where the first inequality is Jensen's inequality and the second uses $q\ge1$. Taking the infimum over $\gamma\in\Gamma(P,P_0)$ yields
\begin{equation*}
\norm{\advreward(x)-\bar r_0(x)}_p
\le
\W_{q,d_p}(P,P_0;x)
\le
\bud.
\end{equation*}

Conversely, suppose $\left\|\advreward(x)-\E_{R\sim P_0}[R(x)]\right\|_p\le \bud$ and write $\Delta\coloneqq\advreward(x)-\E_{R\sim P_0}[R(x)]$. Let $T_\Delta(\widetilde R)(x)\coloneqq\widetilde R(x)+\Delta$ be the promptwise translation map, leaving all other prompts unchanged, and define $P\coloneqq(T_\Delta)_{\#}P_0$ as the pushforward of $P_0$. Then $\E_{R\sim P}[R(x)]=\advreward(x)$. Moreover, the coupling generated by $\widetilde R\sim P_0$ and $R=T_\Delta(\widetilde R)$ has promptwise transport cost
\begin{equation*}
\left(
\E\left[\left\|R(x)-\widetilde R(x)\right\|_p^q\right]
\right)^{1/q}
=
\norm{\Delta}_p
\le
\bud,
\end{equation*}
so $\W_{q,d_p}(P,P_0;x)\le \bud$.
\end{proof}

\section{Proofs for Closed-Form Promptwise Analysis}
\label{app:promptwise-proofs}

\begin{proof}[Proof of \Cref{prop:promptwise-decomposition}]
For compactness, write $\Pi_f\coloneqq\Pi_{\mathrm{full}}$. For a deterministic mean reward $\advreward\in\mathcal M_{\bud,p}^{\mathcal X}(\proxy)$ and a policy $\pi\in\Pi_f$,
\begin{align*}
J_{\advreward}(\beta)-J_{\advreward}(\pi)
&=
\E_{x\sim\mathcal D}
\left[
\ip{\beta(x)-\pi(x)}{\advreward(x)}
\right].
\end{align*}
Because $\Pi_{\mathrm{full}}$ contains all promptwise choices, maximizing over $\beta$ separates over prompts:
\begin{equation*}
\begin{aligned}
\max_{\beta\in\Pi_f}
\left\{J_{\advreward}(\beta)-J_{\advreward}(\pi)\right\}
&=
\E_{x\sim\mathcal D}
\left[
\max_{\beta_x\in\Delta_n}
\ip{\beta_x-\pi(x)}{\advreward(x)}
\right]\\
&=
\E_{x\sim\mathcal D}[\Reg_x(\pi(x),\advreward(x))].
\end{aligned}
\end{equation*}
The ambiguity set is also a promptwise product. Hence, for fixed $\pi$,
\begin{equation*}
\begin{aligned}
\max_{\advreward\in\mathcal M_{\bud,p}^{\mathcal X}(\proxy)}
\max_{\beta\in\Pi_f}
\left\{J_{\advreward}(\beta)-J_{\advreward}(\pi)\right\}
&=
\E_{x\sim\mathcal D}
\left[
\max_{\advreward(x)\in
\mathcal M_{\bud,p}(\proxy(x))}
\Reg_x(\pi(x),\advreward(x))
\right].
\end{aligned}
\end{equation*}
Therefore a global adversary is optimal exactly when it selects an optimal promptwise adversary for $\mathcal D$-almost every prompt; changing the adversary on a null set cannot affect the objective.

Now define
\begin{equation*}
\Phi_x(\zeta)
\coloneqq
\max_{\advreward(x)\in\mathcal M_{\bud,p}(\proxy(x))}
\Reg_x(\zeta,\advreward(x)).
\end{equation*}
Again because the policy can be chosen independently at each prompt,
\begin{equation*}
\min_{\pi\in\Pi_{\mathrm{full}}}\E_{x\sim\mathcal D}[\Phi_x(\pi(x))]
=
\E_{x\sim\mathcal D}\left[\min_{\zeta\in\Delta_n}\Phi_x(\zeta)\right].
\end{equation*}
Thus a global minimizer is precisely a direct product of promptwise minimizers, up to $\mathcal D$-null sets.
\end{proof}

\begin{proof}[Proof of \Cref{prop:inner-adversary}]
For a fixed policy $\pi$ and \(1\le p\le\infty\),
\begin{align*}
\Psi_p(\pi;\proxy)
&=
\max_{\norm{\Delta}_p\le \bud}\ \max_{\beta\in\Delta_n}\ip{\beta-\pi}{\proxy+\Delta}\\
&=
\max_{\beta\in\Delta_n}
\left\{
\ip{\beta-\pi}{\proxy}
+\max_{\norm{\Delta}_p\le \bud}\ip{\beta-\pi}{\Delta}
\right\}\\
&=
\max_{\beta\in\Delta_n}
\left\{
\ip{\beta-\pi}{\proxy}
+\bud\norm{\beta-\pi}_{p^\ast}
\right\},
\end{align*}
where the last step is Holder duality between \(\ell_p\) and \(\ell_{p^\ast}\). The objective is convex in $\beta$, so an optimizer can be chosen at a vertex of the simplex, namely $\beta=e_i$ for some $i$. Evaluating at \(e_i\) gives \eqref{eq:inner-closed-form-lp}.

It remains to record perturbations attaining equality in the dual norm. Let \(k\) be an active vertex. For \(1<p<\infty\) and \(e_k\ne\pi\), equality in Holder's inequality is attained by \eqref{eq:general-p-adversary}. If \(p=\infty\), the constraint is \(\norm{\Delta}_\infty\le\bud\), and choosing \(\Delta_j^\star=\bud\,\operatorname{sgn}((e_k-\pi)_j)\) on nonzero coordinates attains \(\bud\norm{e_k-\pi}_1\).

For \(p=1\), \(p^\ast=\infty\). Since \(\sum_{j\ne i}\pi_j=1-\pi_i\), we have
\[
\norm{e_i-\pi}_\infty
=
\max\left\{1-\pi_i,\max_{j\ne i}\pi_j\right\}
=
1-\pi_i .
\]
Therefore
\begin{align*}
\ip{e_i-\pi}{\proxy}+\bud\norm{e_i-\pi}_\infty
&=
\proxy_i-\ip{\pi}{\proxy}+\bud(1-\pi_i)\\
&=
\bud+\proxy_i-\bud \pi_i-\ip{\pi}{\proxy}.
\end{align*}
Maximizing over $i$ yields \eqref{eq:inner-closed-form}. Equality in the dual norm bound is attained by placing the full $\ell_1$ budget on any maximizer $k\in\argmax_i\{\proxy_i-\bud \pi_i\}$, that is, by choosing $\Delta^\star=\bud e_k$.
\end{proof}

\begin{proof}[Proof of \Cref{prop:lp-epigraph}]
By \Cref{prop:inner-adversary},
\[
\Psi_p(\pi;\proxy)
=
\max_i
\left\{
\proxy_i-\ip{\pi}{\proxy}
+\bud\norm{e_i-\pi}_{p^\ast}
\right\}.
\]
Introducing an epigraph variable \(t\) for the maximum over \(i\) gives
\[
\min_{\pi\in\Delta_n}\Psi_p(\pi;\proxy)
=
\min_{\pi\in\Delta_n,\ t\in\R}
\left\{
t-\ip{\pi}{\proxy}:
t\ge \proxy_i+\bud\norm{e_i-\pi}_{p^\ast},\ i=1,\ldots,n
\right\},
\]
which is exactly \eqref{eq:lp-epigraph}. The feasible set is convex because each map \(\pi\mapsto\norm{e_i-\pi}_{p^\ast}\) is convex and the simplex is convex, while the objective is linear.
\end{proof}

\begin{proof}[Proof of \Cref{thm:water-filling}]
By \Cref{prop:inner-adversary}, minimizing worst-case regret is equivalent to maximizing $F_{\bud}(\pi;\proxy)$ in \eqref{eq:hard-utility}. Introduce an auxiliary variable $t$ and rewrite the problem as the linear program
\begin{equation}
\label{eq:app-lp}
\begin{aligned}
\max_{\pi\in\Delta_n,\ t\in\R}
&\quad \ip{\pi}{\proxy}-t\\
\text{subject to }
&\quad t\ge \proxy_i-\bud \pi_i,\quad i=1,\ldots,n.
\end{aligned}
\end{equation}

Fix $t$. The constraints in \eqref{eq:app-lp} become
\begin{equation*}
\pi_i\ge \frac{(\proxy_i-t)_+}{\bud},\qquad i=1,\ldots,n.
\end{equation*}
Hence a necessary and sufficient feasibility condition for $t$ is
\begin{equation}
\label{eq:feasibility-condition}
\sum_{i=1}^n (\proxy_i-t)_+\le \bud.
\end{equation}
Because the left-hand side of \eqref{eq:feasibility-condition} is continuous and strictly decreasing in $t$, there is a unique minimal feasible threshold $t_0$ solving \eqref{eq:t0-equation}.

For a feasible $t$, maximizing $\ip{\pi}{\proxy}$ subject to the lower bounds is straightforward because $\proxy_1\ge \proxy_2\ge \cdots\ge \proxy_n$. Every coordinate $i\ge 2$ should be set at its minimum feasible value and the remaining probability mass should be assigned to coordinate $1$. Thus
\begin{equation}
\label{eq:pt-form}
\begin{aligned}
\pi_i(t)
&=\frac{(\proxy_i-t)_+}{\bud},
&& i=2,\ldots,n,\\
\pi_1(t)
&=1-\frac{S(t)}{\bud},\\
S(t)
&:=\sum_{i=2}^n(\proxy_i-t)_+.
\end{aligned}
\end{equation}
Substituting \eqref{eq:pt-form} into the objective yields
\begin{equation}
\psi(t)
:=
\ip{\pi(t)}{\proxy}-t
=
\proxy_1-t
-\frac{1}{\bud}\sum_{i=2}^n
(\proxy_1-\proxy_i)(\proxy_i-t)_+.
\end{equation}

The function $\psi$ is piecewise linear. At points $t$ distinct from all reward values,
\begin{equation}
\label{eq:psi-derivative}
\psi'(t)
=
-1+\frac{1}{\bud}\sum_{i:\proxy_i>t}(\proxy_1-\proxy_i).
\end{equation}
As $t$ increases, the active set $\{i:\proxy_i>t\}$ shrinks, so the derivative \eqref{eq:psi-derivative} is nonincreasing. Therefore $\psi$ is concave on $[t_0,\infty)$ and an optimizer is obtained at the smallest feasible threshold whose slope is nonpositive, namely
\begin{equation*}
t^\star=\inf\left\{t\ge t_0:\ \sum_{i:\proxy_i>t}(\proxy_1-\proxy_i)\le \bud\right\}.
\end{equation*}
The optimal policy is then $\pi(t^\star)$, which is exactly \eqref{eq:water-filling-solution}.
\end{proof}

\section{Proofs for Algorithm Section}
\label{app:algorithm-proofs}

\begin{proof}[Proof of \Cref{prop:soft-approx}]
Write
\begin{equation*}
a_i:=\proxy_i-\bud \pi_i,
\qquad
M:=\max_i a_i.
\end{equation*}
Then
\begin{equation*}
e^{M/\tau}
\le
\sum_{i=1}^n e^{a_i/\tau}
\le
n\,e^{M/\tau},
\end{equation*}
so after taking logarithms and multiplying by $\tau$,
\begin{equation*}
M
\le
\tau\log\sum_{i=1}^n e^{a_i/\tau}
\le
M+\tau\log n.
\end{equation*}
Subtracting this inequality from $\ip{\pi}{\proxy}$ gives \eqref{eq:soft-approx}.
\end{proof}

\begin{proof}[Proof of \Cref{prop:soft-gradient}]
Differentiate \eqref{eq:soft-objective} with respect to the promptwise probabilities $\pi_i$:
\begin{align*}
\frac{\partial F_{\bud,\tau}}{\partial \pi_i}
&=
\proxy_i
-\tau\sigma_{\theta,x}(y_i)
\left(-\frac{\bud}{\tau}\right)\\
&=
\proxy_i+\bud\sigma_{\theta,x}(y_i).
\end{align*}
Applying the chain rule and using
\begin{equation*}
\nabla_\theta \pi_{\theta,i}
=
\pi_\theta(y_i\mid x)\nabla_\theta\log\pi_\theta(y_i\mid x)
\end{equation*}
gives \eqref{eq:soft-gradient}.
\end{proof}

\begin{proof}[Proof of \Cref{prop:snis-finite}]
Let
\begin{equation*}
A:=\E_{q}[u(Y)h(Y)],
\qquad
B:=\E_{q}[u(Y)]=\nu_x,
\end{equation*}
so that $\mu_h(x)=A/B$. Define the sample means
\begin{equation*}
\widehat A_K:=\frac{1}{K}\sum_{k=1}^K u^{(k)} h(y^{(k)}),
\qquad
\widehat B_K:=\frac{1}{K}\sum_{k=1}^K u^{(k)},
\end{equation*}
for which $\widehat\mu_{h,K}(x)=\widehat A_K/\widehat B_K$.

Fix $\eta\in(0,1)$ and set
\begin{equation*}
t_K:=\sqrt{\frac{\log(4/\eta)}{2K}}.
\end{equation*}
Because $|u(y)h(y)|\le U_xH_x$ and $0\le u(y)\le U_x$, Hoeffding's inequality gives
\begin{align*}
\Pr\!\left(|\widehat A_K-A|>U_xH_x\,t_K\right)
&\le
2e^{-2Kt_K^2}
=
\frac{\eta}{2},\\
\Pr\!\left(|\widehat B_K-B|>U_x\,t_K\right)
&\le
2e^{-2Kt_K^2}
=
\frac{\eta}{2}.
\end{align*}
Hence, with probability at least $1-\eta$, both bounds hold simultaneously. Under the sample-size condition in the proposition,
\begin{equation*}
U_x t_K
\le
\frac{\nu_x}{2},
\end{equation*}
so on that same event,
\begin{equation*}
\widehat B_K
\ge
B-U_xt_K
\ge
\frac{B}{2}
>
0.
\end{equation*}
Moreover, $|A|\le H_x B$. Therefore
\begin{equation*}
D_K
\coloneqq
\left|\widehat\mu_{h,K}(x)-\mu_h(x)\right|.
\end{equation*}
\begin{align*}
D_K
&=
\left|
\frac{\widehat A_K}{\widehat B_K}
-
\frac{A}{B}
\right|\\
&\le
\frac{|\widehat A_K-A|}{|\widehat B_K|}
+
|A|
\left|
\frac{1}{\widehat B_K}-\frac{1}{B}
\right|\\
&=
\frac{|\widehat A_K-A|}{|\widehat B_K|}
+
\frac{|A|\,|\widehat B_K-B|}{B\,|\widehat B_K|}\\
&\le
\frac{2U_xH_xt_K}{B}
+
\frac{2H_xU_xt_K}{B}
=
\frac{4U_xH_x}{\nu_x}t_K.
\end{align*}
Substituting the definition of $t_K$ proves \eqref{eq:snis-finite-bound}.
\end{proof}

\begin{proof}[Proof of \Cref{prop:dv-budget}]
The Donsker--Varadhan variational formula states that for probability measures $P\ll Q$ and every measurable $f$,
\begin{equation*}
\E_P[f]
\le
\KL(P\,\|\,Q)+\log\E_Q[e^f].
\end{equation*}
Apply this with
\begin{equation*}
P=\pi_\theta(\cdot\mid x),
\qquad
Q=\pi_0(\cdot\mid x),
\qquad
f(y)=\lambda h_x(y),
\end{equation*}
where $\lambda>0$. Then
\begin{equation*}
\lambda\E_{y\sim\pi_\theta(\cdot\mid x)}[h_x(y)]
\le
\KL\big(\pi_\theta(\cdot\mid x)\,\|\,\pi_0(\cdot\mid x)\big)
+
\log\E_{y\sim\pi_0(\cdot\mid x)}[e^{\lambda h_x(y)}].
\end{equation*}
Dividing by $\lambda$ yields \eqref{eq:dv-bound}.
\end{proof}

\begin{proof}[Proof of \Cref{prop:l1-rationale}]
Starting from the definition of regret,
\begin{align*}
\max_{\norm{\Delta}_p\le \bud}\Reg(\pi,\proxy+\Delta)
&=
\max_{\norm{\Delta}_p\le \bud}\ \max_{\beta\in\Delta_n}\ip{\beta-\pi}{\proxy+\Delta}\\
&=
\max_{\beta\in\Delta_n}
\left\{
\ip{\beta-\pi}{\proxy}
+\max_{\norm{\Delta}_p\le \bud}\ip{\beta-\pi}{\Delta}
\right\}\\
&=
\max_{\beta\in\Delta_n}
\left\{
\ip{\beta-\pi}{\proxy}
+\bud\norm{\beta-\pi}_{p^\ast}
\right\},
\end{align*}
where the last step uses the dual norm relation. The map $\beta\mapsto \ip{\beta-\pi}{\proxy}+\bud\norm{\beta-\pi}_{p^\ast}$ is convex, so its maximum over the simplex can be attained at an extreme point $\beta=e_i$. Evaluating the objective at the vertices gives \eqref{eq:lp-dual-rewrite}.

It remains to characterize when $b_{p,i}(\pi)=\norm{e_i-\pi}_{p^\ast}$ is coordinate-local. If $p=1$, then $p^\ast=\infty$. Since $\sum_{j\neq i}\pi_j=1-\pi_i$ and $\pi_j\le 1-\pi_i$ for every $j\neq i$,
\begin{equation*}
\norm{e_i-\pi}_\infty
=
\max\{1-\pi_i,\max_{j\neq i}\pi_j\}
=
1-\pi_i .
\end{equation*}
If $p=\infty$, then $p^\ast=1$, and
\begin{equation*}
\norm{e_i-\pi}_1
=
|1-\pi_i|+\sum_{j\neq i}\pi_j
=
2(1-\pi_i).
\end{equation*}
Thus both endpoint norms yield a bonus depending only on $\pi_i$.

Now suppose $1<p<\infty$, so $1<p^\ast<\infty$. Fix $i=1$ and $a\in(0,1)$. Consider
\begin{equation*}
\pi=(a,1-a,0,\ldots,0)
\end{equation*}
and
\begin{equation*}
\pi'=\left(a,(1-a)/(n-1),\ldots,(1-a)/(n-1)\right).
\end{equation*}
Both policies have the same first coordinate. However,
\begin{equation*}
\norm{e_1-\pi}_{p^\ast}
=
2^{1/p^\ast}(1-a),
\end{equation*}
whereas
\begin{equation*}
\norm{e_1-\pi'}_{p^\ast}
=
(1-a)\left(1+(n-1)^{1-p^\ast}\right)^{1/p^\ast}.
\end{equation*}
For $n\ge3$ and finite $p^\ast>1$, these quantities differ. Hence no function of $\pi_i$ alone can represent $b_{p,i}(\pi)$ for all policies. The displayed $\ell_\infty$ robust-regret formula follows by substituting $\norm{e_i-\pi}_1=2(1-\pi_i)$ into \eqref{eq:lp-dual-rewrite}.
\end{proof}

\section{Additional Experimental Details}
\label{app:experimental-details}

All methods use the same VERL-based RLHF stack with Ray workers and vLLM rollout. We extract the first human turn from \texttt{HuggingFaceH4/hh-rlhf} and format each prompt with the fixed template \texttt{\textbackslash n\textbackslash nHuman: \{human\_turn0\}\textbackslash n\textbackslash nAssistant:}. The prompt pool is shuffled with a fixed seed; 512 prompts are held out for validation, and the remaining prompts are used for training. At validation, the current policy samples completions on the held-out prompts, the same generated responses are scored by the proxy and gold reward models, and the frozen reference policy is used only to compute sequence-level KL. The held-out reward model is used only for evaluation, never for policy updates.

For validation and plotting, the sequence-level KL at checkpoint $t$ is estimated as
\[
\widehat D_{\mathrm{KL},\mathrm{seq}}(\pi_t\|\pi_0)
=\E_{x\sim\mathcal D_{\mathrm{val}},\ y\sim\pi_t(\cdot\mid x)}\left[\sum_{m=1}^{|y|}\log \pi_t(y_m\mid x,y_{<m})-\sum_{m=1}^{|y|}\log \pi_0(y_m\mid x,y_{<m})\right],
\]
where token log-probabilities are summed over the generated response. We also log the corresponding per-token KL, but all reward-versus-KL figures use the sequence-level value.

\paragraph{Existing assets and licenses.}
\path{HuggingFaceH4/hh-rlhf} is a Hugging Face H4 processed prompt source derived from Anthropic HH-RLHF \citep{bai2022training}; its dataset card \url{https://huggingface.co/datasets/HuggingFaceH4/hh-rlhf} lists the MIT license, and the upstream Anthropic HH-RLHF dataset card \url{https://huggingface.co/datasets/Anthropic/hh-rlhf} is also MIT-licensed. We use this dataset only as a prompt distribution and do not redistribute a modified dataset. The policy is Qwen2.5-0.5B-Instruct \citep{qwen2024qwen25,qwen25modelcard}, whose model card lists Apache-2.0. The OpenAssistant proxy reward models are DeBERTaV3-based \citep{he2021debertav3} and their model cards list MIT licenses \citep{openassistant2023rewardbase,openassistant2023rewardlarge}. The held-out evaluator is the existing Tasksource RLHF reward model \citep{sileod2023tasksource}, which is fine-tuned on Anthropic HH-RLHF \citep{bai2022training}. No modified dataset or new model asset is redistributed.

All benchmark methods share the same prompt split, policy initialization, held-out evaluator, generation code path, and logging pipeline within a proxy-quality setting. This is important because the experiment is intended to test reward over-optimization during online policy improvement: the policy repeatedly samples new completions, scores them with the proxy reward model, and is evaluated by a stronger held-out reward model. We always set the explicit KL regularization used by the RLHF loss to zero; KL is computed only as a diagnostic and plotting axis against the frozen initial policy.

\paragraph{Training architecture and run management.}
The training system is built on VERL/HybridFlow \citep{sheng2024hybridflow} with Ray orchestration \citep{moritz2018ray} and vLLM rollout \citep{kwon2023efficient}. Each run has a Ray task runner that builds the prompt datasets, initializes the actor/rollout/reference worker group, and launches the reward-scoring manager. The actor is the only trainable policy model in GRPO-style runs. A frozen reference copy of the initial policy is kept for log-probability and KL diagnostics; it is not optimized. For PPO runs, we additionally enable the critic/value model and use GAE. Rollouts are generated by vLLM in asynchronous mode, while old-policy and reference log probabilities are recomputed by the VERL workers before the policy update. Reward models are not trained during policy optimization. They are loaded by a custom reward manager and scored in batches under no-gradient mode. The proxy reward is used to construct the training reward; the gold reward model is evaluated only in validation/logging and never enters the policy loss.

The later standardized reruns use two NVIDIA RTX A6000 48GB GPUs per run. vLLM uses tensor parallel size $2$ and GPU memory utilization $0.70$. Ray is used to place the actor, reference/log-probability computation, rollout server, and reward-scoring loop in the same run, and each experiment writes a self-contained output directory containing \texttt{config.json}, \texttt{log.csv}, train/validation prompt files, checkpoints, and any auxiliary state such as dynamic-delta estimates. The implementation uses PyTorch, Transformers, PEFT/LoRA, VERL/HybridFlow, Ray, and vLLM; standardized later runs use bfloat16 and \texttt{flash\_attention\_2} when available. The run name records the main experimental choices: method, proxy model, assignment rule, delta or dynamic coefficient, soft-assignment temperature, rollout count, batch/micro-batch size when nondefault, LoRA rank, and attention backend. Runs ending in \texttt{attn-fa2} use \texttt{flash\_attention\_2}; older ablation runs ending in \texttt{attn-sdpa} use PyTorch SDPA attention. This suffix changes the attention kernel used for speed and memory, but not the objective, reward model, or logging definition.

\begin{figure}[t]
\centering
\includegraphics[width=0.9\linewidth]{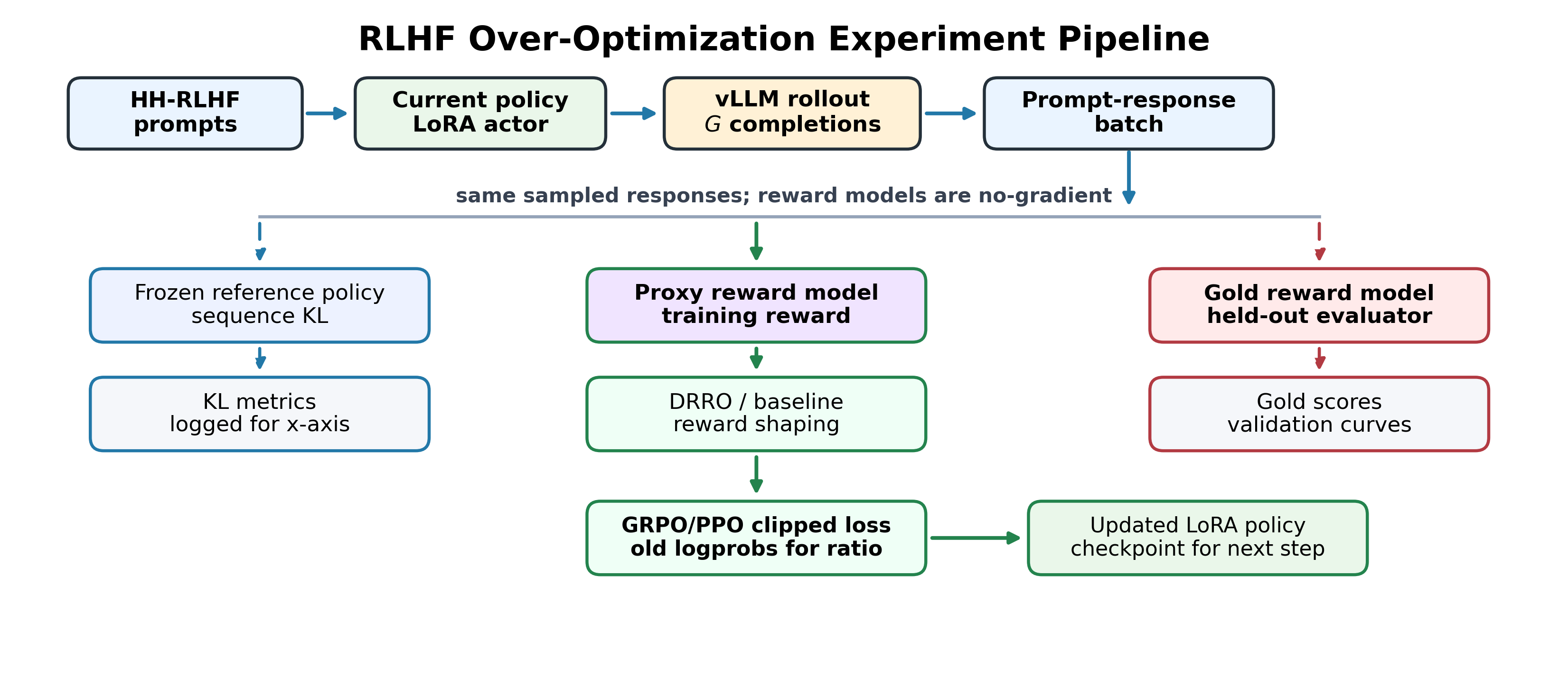}
\caption{Training and evaluation pipeline used in the experiments. VERL and Ray manage the policy, rollout, reference-logprobability, and reward-scoring workers; vLLM serves grouped rollouts, and FlashAttention-2 is used when available. The proxy reward model supplies the training reward and any DRRO/baseline shaping, while the gold reward model is used only for validation curves. The frozen reference policy is used to log sequence KL against the initial policy; no explicit KL penalty is added to the policy loss.}
\label{fig:training-pipeline}
\end{figure}

\paragraph{Policy-optimization hyperparameters.}
The standardized \texttt{attn-fa2} reruns use \path{Qwen/Qwen2.5-0.5B-Instruct} as the policy. We train LoRA adapters with rank $8$ and LoRA alpha $16$ on all linear layers, with gradient checkpointing enabled. The actor optimizer is AdamW with learning rate $10^{-5}$, weight decay $0.01$, betas $(0.9,0.999)$, a constant learning-rate schedule with no warmup, and gradient clipping at $1.0$. Actor and reference computations use FSDP with bfloat16, and rollout generation uses vLLM with bfloat16. Each prompt produces $G=16$ completions unless the run name explicitly states another rollout count. The maximum prompt length is $512$ tokens and the maximum response length is $128$ tokens. Sampling uses temperature $1.0$, top-$p=0.95$, and top-$k=-1$.

The standardized later runs use train and validation batch size $16$ prompts, actor PPO minibatch size $16$, and actor, old-log-probability, and reference-log-probability micro-batch size $8$ per GPU. Reward-model scoring uses batch size $16$ and maximum reward-model input length $512$. All policy-optimization runs use $300$ policy updates, validation before training, validation every $5$ updates, and checkpointing every $20$ updates. Each validation pass samples from the current policy on the held-out prompt split and records proxy reward, gold reward, sequence-level KL, and per-token KL on those validation completions. Logged quantities include raw proxy and gold scores, re-centered proxy and gold scores, sequence-level KL, and per-token KL. The score re-centering baseline is computed at step $0$ separately for proxy and gold, so every plotted reward curve is an improvement over the initial policy under that evaluator's own scale. The explicit KL regularization coefficient in the policy loss is always $0.0$; the KL values are kept only to measure distance from the frozen initial policy and to construct the reward-versus-KL plots. Some earlier exploratory runs used batch size $12$, rollout count $8$, or SDPA attention; these differences are encoded in the run names and were not used as the standardized later setting.

\paragraph{Baseline implementations.}
GRPO uses the proxy reward directly and converts the group of $G$ completions for a prompt into normalized advantages by subtracting the group mean and dividing by the group standard deviation. PPO uses the same proxy reward, but adds a value function and GAE with $\lambda=0.95$. The PPO critic is initialized from the same policy backbone with a value head and is trained by the VERL FSDP critic worker, using critic micro-batch size $1$ per GPU; it is used only for value estimation and advantage computation, while the reward still comes from the proxy reward model. Both PPO and GRPO use the clipped surrogate with clip radius $0.2$, one PPO epoch per rollout batch, entropy coefficient $0$, and no KL term in the reward or actor loss. For GRPO-style methods, the advantage estimator has $\lambda=1$ because the completion-level reward is broadcast over the generated response tokens after group normalization.

For ensemble baselines, we trained and loaded five proxy reward models and tested three aggregation rules. The five members are initialized from small encoder reward-model backbones, \path{microsoft/MiniLM-L12-H384-uncased}, \path{prajjwal1/bert-small}, \path{google/electra-small-discriminator}, \path{distilbert-base-uncased}, and \path{distilroberta-base}, and are trained on the same gold-labeled pair data with seeds $42,43,44,45,46$. Each member uses one epoch, learning rate $2\times10^{-5}$, batch size $16$, maximum sequence length $512$, no weight decay, and gradient accumulation $1$. Because the members have different scalar reward calibrations, we apply the running per-member z-score calibration implemented in the reward manager before aggregation. The mean ensemble uses $R_{\mu}(x,y)=k^{-1}\sum_{j=1}^k R_j(x,y)$. The worst-case ensemble uses $R_{\mathrm{WCO}}(x,y)=\min_j R_j(x,y)$. The uncertainty-weighted ensemble uses $R_{\mathrm{UWO}}(x,y)=k^{-1}\sum_j R_j(x,y)-\lambda\operatorname{Var}_j(R_j(x,y))$ with $\lambda=1.0$. The main comparison reports Ensemble-Mean and Ensemble-UWO because mean is the natural nonconservative ensemble baseline and UWO is the uncertainty-penalized baseline. WCO is implemented and was used as an additional conservative check. During policy optimization, the policy sees only the aggregated ensemble reward.

BSPO and InfoRM are included as reward/behavior-support baselines rather than as new proxy-quality settings. BSPO uses the same policy proxy reward as the corresponding experiment plus a trained ScoreLM model, \texttt{scorelm\_qwen0p5b\_lmcoef0.01}. The ScoreLM is initialized from \path{Qwen/Qwen2.5-0.5B-Instruct} and trained on the same preference-pair data with a Bradley--Terry score loss plus an auxiliary language-modeling loss, weighted by \texttt{lm\_coef}$=0.01$; its training uses two epochs, learning rate $2\times10^{-5}$, batch size $8$, maximum length $1024$, weight decay $0.1$, and seed $42$. In the policy run, the scalar reward remains the shared proxy RM score. ScoreLM supplies token-level behavior-support masks with maximum ScoreLM length $1024$ and support threshold $\epsilon_\beta=10^{-4}$; unsupported actions receive the pessimistic target value $-15$ in the BSPO advantage computation. InfoRM uses the same policy proxy reward plus an auxiliary InfoRM model with MiniLM-L12-H384 encoder, latent dimension $128$, $\beta=0.01$, and penalty coefficient $0.01$. The InfoRM model is trained for one epoch on the preference pairs with a Bradley--Terry loss plus the variational information-bottleneck KL term, learning rate $5\times10^{-6}$, batch size $16$, maximum length $512$, CLS pooling, and seed $42$. During policy optimization, InfoRM keeps the raw proxy score for logging but uses the shaped reward $r_{\mathrm{proxy}}-0.01\,\mathrm{KL}(q(z\mid x,y)\|N(0,I))$. Thus in both baselines the same proxy reward remains the main RLHF signal, while the auxiliary model changes the support or penalty component of the training signal.

\paragraph{DRRO, DRO, and grid search.}
For DRRO, hard assignment follows the stabilized rule in \eqref{eq:practical-hard-bonus}: after group-normalizing rollout probabilities, it selects the sampled completion with largest $\widehat r^{(k)}-\widetilde\delta\,\widetilde\pi^{(k)}$ and applies the robust correction to that completion before advantage normalization. Soft assignment replaces this single winner by the smooth SNIS-style bonus in \eqref{eq:practical-soft-bonus}. In the fixed-budget setting, the standardized value is $\delta=40$ when $G=16$, corresponding to $2.5G$; additional ablation runs use the fixed deltas stated in their run names. In the dynamic-budget setting, the budget is $\delta_t=\alpha\widehat{\KL}_t$, where $\widehat{\KL}_t$ is the $k_3$ KL estimate smoothed over a sliding window of $20$ updates. We do not clip the dynamic budget with an explicit minimum or maximum in the reported runs. DRO uses the same rollout, reward, and logging stack but applies the value-robust correction instead of the regret-robust correction; its fixed delta is the value stated in the run name.

Before selecting the reported DRRO configuration, we ran a small grid over the two design choices in the practical algorithm: whether the robust budget is fixed or dynamic, and whether the robust correction is hard-assigned to one sampled response or softly distributed across the sampled group. For dynamic-hard runs, we swept the dynamic coefficient $\alpha\in\{0.5,1,2,5,10\}$. For dynamic-soft runs, we swept the same $\alpha$ values and soft-assignment temperatures $\tau\in\{1,2,5,10\}$. For fixed-hard and fixed-soft runs, we swept fixed budgets including $\delta\in\{10,20\}$ in the grid, and also checked the larger standardized budget $\delta=40=2.5G$ used in the later matched reruns; fixed-soft additionally swept $\tau\in\{1,2,5,10\}$. The grid showed that dynamic budgeting is more stable than a single fixed budget across KL regimes, while soft assignment reduces the sensitivity of the update to one noisy sampled completion. Based on this sweep, the displayed standardized dynamic-soft setting uses $\alpha=10$ and $\tau=2.0$ for the later \texttt{attn-fa2} reruns. The ablation figure reports the qualitative effect of these choices: dynamic budgeting improves the frontier relative to fixed budgeting, and soft assignment gives the strongest version among the sampled DRRO variants. For the plotted ablation curves in \Cref{fig:ablation}, fixed-hard uses $\delta=40$ with hard assignment; fixed-soft uses $\delta=40$ and $\tau=2.0$; dynamic-hard uses $\alpha=10$ with hard assignment; and dynamic-soft uses $\alpha=10$ and $\tau=2.0$.

\paragraph{Proxy-reward agreement and fine-tuning.}
All proxy-quality comparisons use the same policy and held-out evaluator: Qwen2.5-0.5B-Instruct is the initial policy, and the Tasksource DeBERTa-v3-large RLHF reward model is the gold evaluator. The policy has roughly 0.5B parameters. The small proxy is an OpenAssistant DeBERTa-v3-base reward model with roughly 184M parameters, while the larger proxy is an OpenAssistant DeBERTa-v3-large-v2 reward model with roughly 435M parameters. The gold reward model is a separate Tasksource RLHF evaluator and is not used as a policy-training proxy. We run each policy-optimization setting with five random seeds, $100,200,300,400,500$, and use the same prompt split, rollout settings, and evaluation pipeline across seeds.

We measure proxy--gold agreement on an evaluation split of 5,000 held-out response pairs. For each pair $(y_a,y_b)$ under the same prompt $x$, both reward models score the two responses independently. Agreement is the pairwise ranking match
\[
\frac{1}{N}\sum_{i=1}^N
\ind\left\{
\begin{aligned}
&\operatorname{sign}\big(r_{\mathrm{proxy}}(x_i,y_{i,a})-r_{\mathrm{proxy}}(x_i,y_{i,b})\big)\\
&\qquad=
\operatorname{sign}\big(r_{\mathrm{gold}}(x_i,y_{i,a})-r_{\mathrm{gold}}(x_i,y_{i,b})\big)
\end{aligned}
\right\}.
\]
Thus agreement measures whether the proxy and gold reward models prefer the same response in a pair; it does not require the two scalar reward outputs to have the same calibration or numerical scale.

The fine-tuned proxy used in the main benchmark is initialized from the OpenAssistant DeBERTa-v3-base reward model and trained on gold-labeled preference pairs. We construct 20,000 pairwise comparisons from \path{HuggingFaceH4/hh-rlhf} prompts by sampling two candidate responses from the policy-generation pipeline, scoring both with the gold reward model, and assigning the higher-gold response as preferred. The pairs are split deterministically into train and validation subsets. The proxy is optimized for one epoch with a Bradley--Terry pairwise ranking loss, learning rate $2\times10^{-5}$, batch size 16, maximum sequence length 512, no weight decay, and gradient accumulation 1. The validation split is used only to report agreement. After this reward-model fine-tuning step, policy optimization still uses only the proxy reward; the gold reward model remains held out for evaluation.

\paragraph{Main benchmark: fine-tuned small proxy.}
\Cref{fig:main-benchmark} uses the fine-tuned small proxy, i.e., the DeBERTa-v3-base proxy after the gold-labeled preference fine-tuning described above. Its proxy--gold agreement on the 5,000-pair evaluation set is 85.7\%. This is the intermediate regime: the proxy is substantially better than the raw small proxy, but it is still imperfect. In the figure, the dashed proxy trajectory continues to increase as optimization proceeds, while several gold curves peak and then flatten or decline. This remaining proxy--gold gap is the regime where robust policy updates are meaningful. DRRO benefits from the cleaner proxy signal but still uses regret robustness to avoid blindly chasing high-proxy responses that are not well supported by the sampled policy distribution. The plotted shaded regions summarize the five-seed runs with seeds $100,200,300,400,500$.

\begin{figure}[t]
\centering
\includegraphics[width=0.86\linewidth]{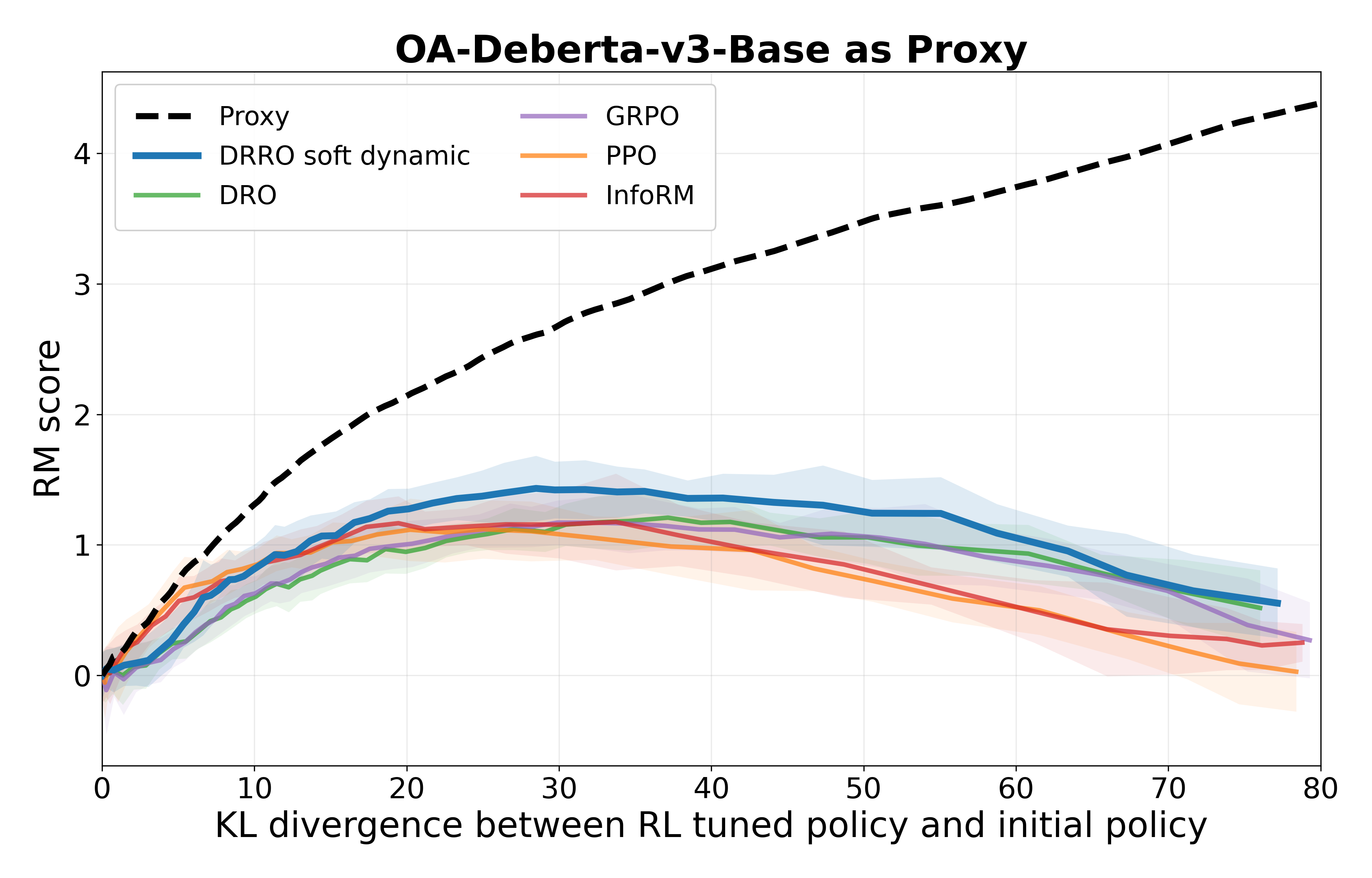}
\caption{Raw-proxy benchmark using the OpenAssistant DeBERTa-v3-base reward model without the proxy fine-tuning stage. The raw proxy has 73.6\% pairwise agreement with the held-out gold reward model on the 5,000-pair evaluation set, compared with 85.7\% for the fine-tuned proxy used in the main benchmark. The dashed curve is the averaged proxy trajectory, and shaded regions show mean $\pm$ one standard deviation across the five seeds.}
\label{fig:raw-proxy-appendix}
\end{figure}

\paragraph{Raw-proxy benchmark.}
\Cref{fig:raw-proxy-appendix} uses the same small proxy architecture as the main benchmark, OpenAssistant DeBERTa-v3-base, but without any gold-labeled fine-tuning. Its agreement with the gold evaluator drops to 73.6\%. This lower agreement means that the proxy ranking is much noisier: responses that look good to the proxy are often not preferred by the gold evaluator. The figure reflects this mismatch. The proxy curve remains optimistic, whereas the gold curves separate more sharply across methods and show stronger over-optimization behavior. In this low-agreement setting, DRRO has more room to help because the regret correction is explicitly designed for proxy misspecification. Compared with the main benchmark, the raw-proxy setting is therefore a harder stress test: it asks whether the policy update can remain useful when the reward signal is substantially less aligned with the held-out evaluator.

\begin{figure}[t]
\centering
\includegraphics[width=0.86\linewidth]{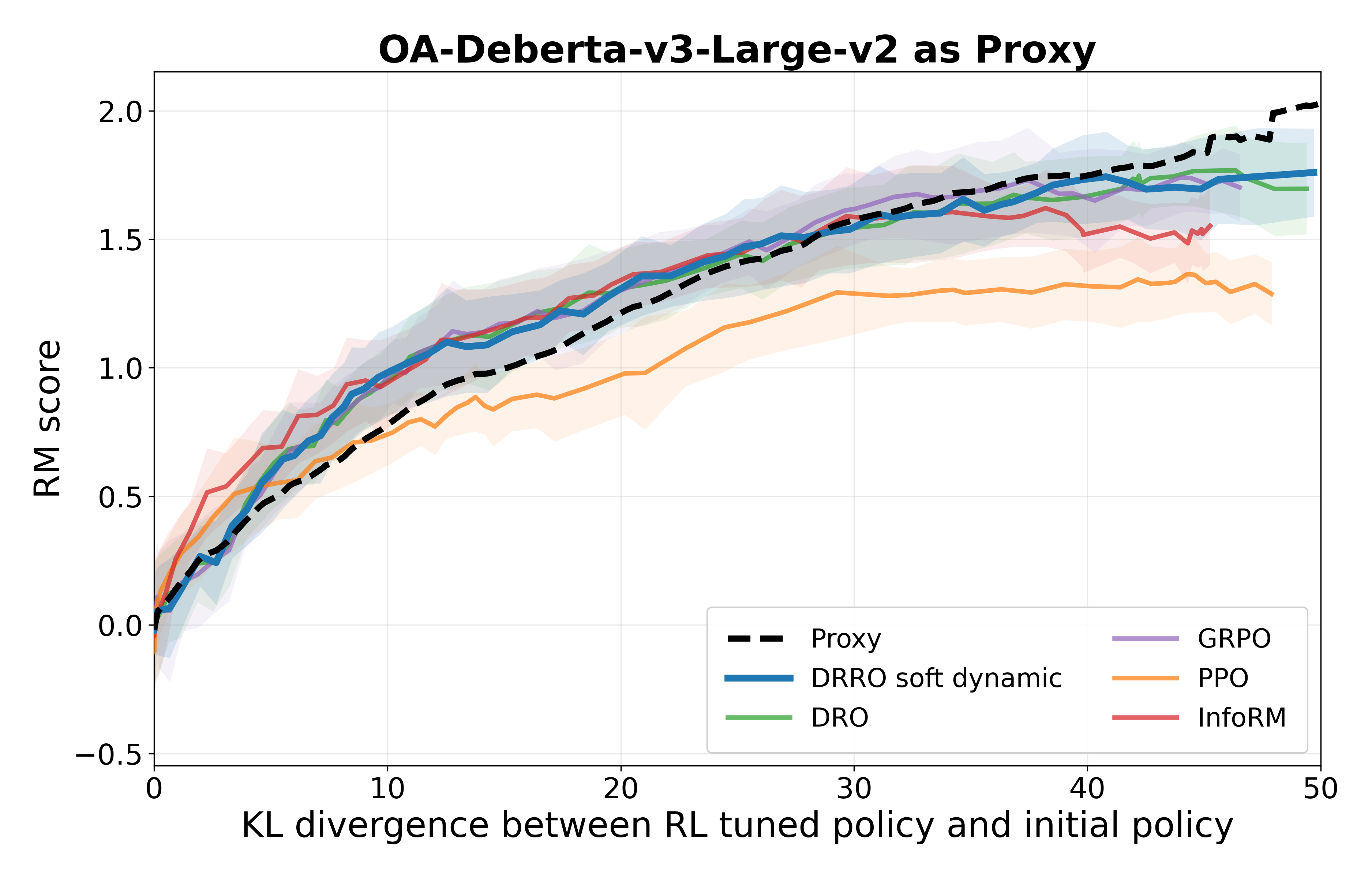}
\caption{Large-proxy benchmark using the OpenAssistant DeBERTa-v3-large-v2 reward model as the policy-training proxy. This larger proxy has 95.2\% pairwise agreement with the held-out gold reward model on the same 5,000-pair evaluation protocol. The dashed curve is the averaged proxy trajectory, and shaded regions show mean $\pm$ one standard deviation across the five seeds.}
\label{fig:large-proxy-appendix}
\end{figure}

\paragraph{Large-proxy benchmark.}
\Cref{fig:large-proxy-appendix} replaces the small proxy with OpenAssistant DeBERTa-v3-large-v2. This proxy is not the fine-tuned DeBERTa-v3-base model from the main benchmark; instead, it is a larger DeBERTa-v3-large reward model used directly as the policy-training proxy. Its agreement with gold is 95.2\%, much higher than both the raw small proxy and the fine-tuned small proxy. The plot changes accordingly. Since the proxy ranking is already very close to the gold ranking, optimizing the proxy is a much better approximation to optimizing the held-out evaluator. The gold curves of PPO, GRPO, ensemble baselines, DRO, and DRRO become closer to each other, and the advantage of a robust correction is less visually pronounced. This does not contradict the raw-proxy or main benchmark results; it shows that when proxy misspecification is small, all reasonable update rules see a similar reward signal.

\paragraph{Comparison across proxy qualities.}
Across the three settings, the qualitative trend is consistent with the role of DRRO as a misspecification-robust policy update. With the raw small proxy, agreement is only 73.6\%, so proxy optimization is unreliable and over-optimization is most visible. With the fine-tuned small proxy in the main benchmark, agreement improves to 85.7\%; the proxy is useful but still imperfect, and DRRO achieves its clearest frontier improvement. With the larger proxy, agreement reaches 95.2\%; proxy and gold are nearly rank-aligned, so the methods naturally cluster together. The comparison indicates that DRRO is most valuable when the proxy is informative but not fully aligned with the held-out evaluator. If the proxy is very weak, all methods are limited by reward noise; if the proxy is nearly gold-aligned, there is little misspecification left for any robustness method to correct.

\subsection{Hard-Max DRRO Ablation Algorithm}
\label{app:hard-max-ablation}

The hard-max ablation is the most direct sampled translation of the exact \(\ell_1\) promptwise DRRO objective. For a fixed prompt,
\begin{equation*}
F_{\bud}(\pi;\proxy)=\ip{\pi}{\proxy}-\max_i(\proxy_i-\bud \pi_i),
\end{equation*}
which differs from worst-case regret only by an additive constant. Define \(H_{\bud}(\pi;\proxy):=-\max_i(\proxy_i-\bud \pi_i)\). If the maximizer is unique, differentiating this additional term gives
\begin{equation*}
\nabla_\theta H_{\bud}(\pi_\theta(\cdot\mid x);\proxy(x))
=
\sum_{i=1}^n \pi_\theta(y_i\mid x)
\Big[\bud\,\ind\{i=i^\star(\pi_\theta)\}\Big]
\nabla_\theta\log\pi_\theta(y_i\mid x),
\end{equation*}
where
\begin{equation*}
i^\star(\pi_\theta)\in\argmax_{1\le i\le n}\{\proxy_i-\bud \pi_{\theta,i}\}.
\end{equation*}

\begin{proposition}[Gradient of the hard promptwise objective]
\label{prop:hard-gradient}
Fix a prompt \(x\) and suppose
\[
i^\star(\pi_\theta)\in\argmax_{1\le i\le n}\{\proxy_i-\bud \pi_{\theta,i}\}
\]
is unique. Then
\begin{equation}
\label{eq:hard-gradient}
\begin{aligned}
g_{\theta}^{\mathrm{hard}}(x)
&:=
\nabla_\theta F_{\bud}
(\pi_\theta(\cdot\mid x);\proxy(x)),\\
g_{\theta}^{\mathrm{hard}}(x)
&=
\sum_{i=1}^n \pi_\theta(y_i\mid x)
R_i^{\mathrm{hard}}
\nabla_\theta\log\pi_\theta(y_i\mid x),\\
R_i^{\mathrm{hard}}
&:=
\proxy_i+\textcolor{blue!70!black}{B_i^{\mathrm{hard}}},\\
B_i^{\mathrm{hard}}
&:=
\bud\,\ind\{i=i^\star(\pi_\theta)\}.
\end{aligned}
\end{equation}
\end{proposition}

\begin{proof}
Let
\begin{equation*}
a_i(\pi):=\proxy_i-\bud \pi_i,
\qquad
F_{\bud}(\pi;\proxy)=\ip{\pi}{\proxy}-\max_i a_i(\pi).
\end{equation*}
If the maximizer $i^\star(\pi)$ is unique, then the maximum is differentiable at $\pi$ and
\begin{equation*}
\frac{\partial F_{\bud}}{\partial \pi_i}
=
\proxy_i-\frac{\partial a_{i^\star(\pi)}(\pi)}{\partial \pi_i}
=
\proxy_i+\bud\,\ind\{i=i^\star(\pi)\}.
\end{equation*}
Applying the chain rule with
\begin{equation*}
\nabla_\theta \pi_{\theta,i}
=
\pi_\theta(y_i\mid x)\nabla_\theta\log\pi_\theta(y_i\mid x)
\end{equation*}
yields \eqref{eq:hard-gradient}.
\end{proof}

\Cref{prop:hard-gradient} shows that hard DRRO is still a reward-shaping rule: the adversarially threatening response receives a budget-sized bonus. Since the full maximizer over the hidden simplex is unobserved, the ablation replaces it by the maximizer in the sampled group. For prompt \(x\), write
\[
\bar\pi^{(k)}:=\rollpi(y^{(k)}\mid x),
\qquad
\widehat r^{(k)}:=\proxy(x,y^{(k)}),
\]
and define the sampled hard winner
\begin{equation}
\label{eq:hard-winner}
k^\star\in\argmax_{1\le k\le G}\{\widehat r^{(k)}-\bud\,\bar\pi^{(k)}\}.
\end{equation}
The conceptual hard DRRO-shaped rewards are
\begin{equation}
\label{eq:hard-bonus}
\widetilde r_{\mathrm{hard}}^{(k)}
=
\widehat r^{(k)}+\textcolor{blue!70!black}{\bud\,\ind\{k=k^\star\}}.
\end{equation}
For the implemented ablation, we use the same group-normalized probability coordinates as the final soft DRRO algorithm. That is, with
\[
\widetilde\pi^{(k)}
\coloneqq
\frac{\bar\pi^{(k)}}{\sum_{\ell=1}^G\bar\pi^{(\ell)}}
\qquad\text{and}\qquad
\widetilde{\bud}\coloneqq(G/n)\bud,
\]
the practical hard DRRO bonus is
\begin{equation}
\label{eq:practical-hard-bonus}
\begin{aligned}
\widetilde k^\star
&\in
\argmax_{1\le k\le G}
\{\widehat r^{(k)}
-\widetilde{\bud}\,\widetilde\pi^{(k)}\},\\
\widetilde r_{\mathrm{hard}}^{(k)}
&\coloneqq
\widehat r^{(k)}
+\textcolor{blue!70!black}{
\widetilde{\bud}\,\ind\{k=\widetilde k^\star\}}.
\end{aligned}
\end{equation}
The algorithm then normalizes these shaped rewards within each prompt group and plugs the resulting advantages into the same clipped GRPO surrogate \eqref{eq:grpo-loss}.

\begin{algorithm}[t]
\caption{Hard-max DRRO-RLHF ablation (blue lines differ from \Cref{alg:grpo})}
\label{alg:hard-drro-rlhf}
\begin{algorithmic}[1]
\Require initial policy $\pi_\theta$, prompt source $\mathcal D$, reward model $\proxy$, budget rule for $\widetilde{\bud}_b$, outer iterations $M$, prompt batch size $B$, group size $G$, clip radius $\varepsilon_{\mathrm{clip}}$, policy-gradient steps $S$
\Ensure updated policy $\pi_\theta$
\For{outer iteration $m=1,\ldots,M$}
  \State set rollout policy $\rollpi\gets\pi_\theta$
  \State sample prompt batch $x_1,\ldots,x_B\sim\mathcal D$
  \For{each prompt $x_b$ in the batch}
  \State sample $y_b^{(1)},\ldots,y_b^{(G)}\sim\rollpi(\cdot\mid x_b)$
  \State compute proxy rewards $\widehat r_b^{(k)}=\proxy(x_b,y_b^{(k)})$
  \State instantiate scaled budget $\widetilde{\bud}_b$ from the fixed or dynamic budget rule
  \State \algchange{compute rollout probabilities $\bar\pi_b^{(k)}=\rollpi(y_b^{(k)}\mid x_b)$ and normalized probabilities $\widetilde\pi_b^{(k)}$}
  \State \algchange{set hard shaped rewards $\widetilde r_{b,\mathrm{hard}}^{(k)}$ according to \eqref{eq:practical-hard-bonus}}
  \State normalize \algchange{$\widetilde r_{b,\mathrm{hard}}^{(k)}$} within the prompt group to form advantages
  \State compute response ratios $\rho_b^{(k)}(\theta)=\pi_\theta(y_b^{(k)}\mid x_b)/\rollpi(y_b^{(k)}\mid x_b)$
  \State accumulate the clipped GRPO surrogate \eqref{eq:grpo-loss}
  \EndFor
  \State update $\theta$ using $S$ minibatch policy-gradient steps on \eqref{eq:grpo-loss}
\EndFor
\end{algorithmic}
\end{algorithm}

\section{Proof of DRRO-Versus-DRO Comparison Theorem}
\label{app:drro-vs-dro}

\begin{lemma}[Explicit DRO optimizer]
\label{lem:dro-explicit}
For each $k\in\{1,\ldots,n\}$, define
\begin{equation*}
u^{(k)}:=
\left(\underbrace{\tfrac1k,\ldots,\tfrac1k}_{k\text{ entries}},0,\ldots,0\right),
\qquad
A_k:=\frac{1}{k}\sum_{i=1}^k \proxy_i.
\end{equation*}
Then a DRO optimizer can be chosen as $u^{(m)}$ for some
\begin{equation*}
m\in\argmax_{1\le k\le n}\left\{A_k-\frac{\bud}{k}\right\}.
\end{equation*}
\end{lemma}

\begin{proof}
Fix $m=\norm{\pi}_\infty$. Maximizing $\ip{\pi}{\proxy}$ subject to $\pi_i\le m$, $\sum_i \pi_i=1$, and $\pi_i\ge 0$ is done by filling the largest rewards first, so an optimizer has the form
\begin{equation*}
(m,\ldots,m,1-(k-1)m,0,\ldots,0)
\end{equation*}
for some $k$. On each such region the objective $\ip{\pi}{\proxy}-\bud m$ is linear in $m$, so an optimum occurs at an endpoint, which is exactly $u^{(k)}$ for some $k$. Evaluating the objective at $u^{(k)}$ gives $A_k-\bud/k$.
\end{proof}

\begin{proof}[Proof of \Cref{thm:drro-dominates-dro}]
Let $\pi^{\mathrm{DRRO}}=\pi^{\mathrm{DRRO}}(\bud)$ and let $k$ be the number of positive entries of $\pi^{\mathrm{DRRO}}$. By \Cref{thm:water-filling}, the DRRO optimizer is supported on the top $k$ actions and is nonincreasing:
\begin{equation*}
\pi_1^{\mathrm{DRRO}}\ge \pi_2^{\mathrm{DRRO}}\ge \cdots \ge \pi_k^{\mathrm{DRRO}}>0,
\qquad
\pi_{k+1}^{\mathrm{DRRO}}=\cdots=\pi_n^{\mathrm{DRRO}}=0.
\end{equation*}
Because $\pi_k^{\mathrm{DRRO}}>0$, \Cref{eq:water-filling-solution} implies $\proxy_k>t^\star\ge t_0$. Hence
\begin{equation*}
\bud
\;>\;
\sum_{i=1}^k (\proxy_i-\proxy_k)
\;=\;
k(A_k-\proxy_k).
\end{equation*}
More generally, since $\proxy_j>t_0$ for every $j\le k$,
\begin{equation}
\label{eq:dro-support-monotone}
\bud
\;>\;
\sum_{i=1}^j(\proxy_i-\proxy_j)
\;=\;
j(A_j-\proxy_j),
\qquad j=2,\ldots,k.
\end{equation}

Now let $m$ be the support size of the DRO optimizer supplied by \Cref{lem:dro-explicit}. Define
\begin{equation*}
v_j:=A_j-\frac{\bud}{j},\qquad j=1,\ldots,n.
\end{equation*}
For each $j\ge 2$,
\begin{equation*}
v_j-v_{j-1}
=
\frac{\bud-j(A_j-\proxy_j)}{j(j-1)}.
\end{equation*}
By \eqref{eq:dro-support-monotone}, $v_j-v_{j-1}>0$ for all $j\le k$, so
\begin{equation*}
v_1< v_2<\cdots < v_k.
\end{equation*}
Therefore any maximizer $m$ of $v_j$ must satisfy $m\ge k$.

Since $\pi^{\mathrm{DRRO}}$ is decreasing and supported on the first $k$ coordinates, for every $j\le k$,
\begin{equation}
\label{eq:prefix-lower-bound}
\sum_{i=1}^j \pi_i^{\mathrm{DRRO}}
\ge
\frac{j}{k}
\ge
\frac{j}{m}
=
\sum_{i=1}^j u_i^{(m)}.
\end{equation}
The first inequality in \eqref{eq:prefix-lower-bound} follows because the average of the first $j$ entries of a decreasing $k$-support probability vector is at least the average of all $k$ supported entries. For $j\ge k$, the left-hand side equals $1$, while the right-hand side is at most $1$. Hence
\begin{equation}
\label{eq:prefix-dominance}
\sum_{i=1}^j \pi_i^{\mathrm{DRRO}}
\ge
\sum_{i=1}^j \pi_i^{\mathrm{DRO}},
\qquad j=1,\ldots,n-1.
\end{equation}

Because $\phi$ is increasing and $\proxy_1>\cdots>\proxy_n$, the true rewards satisfy
\begin{equation*}
t_1>t_2>\cdots>t_n.
\end{equation*}
Let $d_i:=\pi_i^{\mathrm{DRRO}}-\pi_i^{\mathrm{DRO}}$. Then $\sum_i d_i=0$ and \eqref{eq:prefix-dominance} states that $\sum_{i=1}^j d_i\ge 0$ for all $j<n$. Abel summation gives
\begin{equation*}
\sum_{i=1}^n d_i t_i
=
\sum_{j=1}^{n-1}\left(\sum_{i=1}^j d_i\right)(t_j-t_{j+1}).
\end{equation*}
Every term on the right-hand side is nonnegative, so
\begin{equation*}
\sum_{i=1}^n \pi_i^{\mathrm{DRRO}}\,t_i
\ge
\sum_{i=1}^n \pi_i^{\mathrm{DRO}}\,t_i.
\end{equation*}
If $\phi$ is strictly increasing and the two policies differ, then at least one prefix inequality in \eqref{eq:prefix-dominance} is strict, and every increment $t_j-t_{j+1}$ is positive, so the final inequality is strict as well.
\end{proof}

\section{Proofs for Coverage-Based Comparison Results}
\label{app:coverage-proofs}

\begin{proof}[Proof of \Cref{prop:dro-certificate}]
Let $e=r-\proxy$ and write $C_\pi\coloneqq C_\infty(\pi;\mu,\mathcal D)$. Then
\begin{equation*}
\begin{aligned}
J_r(\pi)
&=
J_{\proxy}(\pi)
+\E_{x\sim\mathcal D}
\big[\ip{z_\pi(x)}{e(x)}\big].
\end{aligned}
\end{equation*}
Weighted $\ell_1$--$\ell_\infty$ H\"older's inequality on the product prompt-response space gives
\begin{equation*}
\begin{aligned}
\left|\E_{x}\big[\ip{z_\pi(x)}{e(x)}\big]\right|
&\le
\E_x\big[
\norm{z_\pi(x)}_{\infty,\mu_x^{-1}}
\norm{e(x)}_{1,\mu_x}
\big]\\
&\le
C_\pi\,
\E_x\big[\norm{e(x)}_{1,\mu_x}\big].
\end{aligned}
\end{equation*}
Thus every $r\in\U^{(1)}_\varepsilon$ satisfies
\begin{equation*}
\E_{x}\big[\ip{z_\pi(x)}{e(x)}\big]\ge
-\varepsilon C_\pi.
\end{equation*}
Hence
\begin{equation*}
\begin{aligned}
\min_{r\in\U^{(1)}_\varepsilon}J_r(\pi)
&\ge
J_{\proxy}(\pi)
-\varepsilon C_\pi.
\end{aligned}
\end{equation*}
The reverse inequality is obtained by concentrating the perturbation on prompt-response coordinates whose ratio $|z_\pi(x,y)|/\mu_x(y)$ approaches the essential supremum. More explicitly, for any $\eta>0$, let $S_\eta$ be a positive-probability set of prompts on which $\norm{z_\pi(x)}_{\infty,\mu_x^{-1}}\ge C_\infty(\pi;\mu,\mathcal D)-\eta$, choose a maximizing response $y_\pi^\star(x)$ on that set, and put all $\ell_1$ budget on those coordinates with sign opposite to $z_\pi(x,y_\pi^\star(x))$. This gives rewards in $\U^{(1)}_\varepsilon$ whose objective values are at most $J_{\proxy}(\pi)-\varepsilon(C_\infty(\pi;\mu,\mathcal D)-\eta)$. Letting $\eta\downarrow0$ proves \eqref{eq:dro-certificate}.
\end{proof}

\begin{proof}[Proof of \Cref{prop:dro-regret}]
Because $\true\in\U^{(1)}_\varepsilon$,
\begin{equation*}
J_{\true}(\pi_\rho^{\mathrm{DRO}})
\ge
\min_{r\in\U^{(1)}_\varepsilon}J_r(\pi_\rho^{\mathrm{DRO}})
\ge
\min_{r\in\U^{(1)}_\varepsilon}J_r(\pi_\rho^\star),
\end{equation*}
where the last step uses the optimality of $\pi_\rho^{\mathrm{DRO}}$ over $\Pi_\rho$. Applying \Cref{prop:dro-certificate} to $\pi_\rho^\star$ yields
\begin{equation*}
J_{\true}(\pi_\rho^{\mathrm{DRO}})
\ge
J_{\proxy}(\pi_\rho^\star)-\varepsilon C_\infty(\pi_\rho^\star;\mu,\mathcal D).
\end{equation*}
On the other hand, because $\true\in\U^{(1)}_\varepsilon$,
\begin{equation*}
J_{\true}(\pi_\rho^\star)
\le
J_{\proxy}(\pi_\rho^\star)+\varepsilon C_\infty(\pi_\rho^\star;\mu,\mathcal D).
\end{equation*}
Subtracting the two inequalities gives \eqref{eq:dro-regret-bound}.
\end{proof}

\begin{proof}[Proof of \Cref{prop:drro-certificate}]
For fixed $\pi$ and $\beta$,
\[
C_{\beta,\pi}
\coloneqq
C_{\infty,\mathrm{rel}}(\beta,\pi;\mu,\mathcal D).
\]
\begin{equation*}
J_r(\beta)-J_r(\pi)
=
\begin{aligned}[t]
&J_{\proxy}(\beta)-J_{\proxy}(\pi)\\
&\quad+
\E_x\big[\ip{z_\beta(x)-z_\pi(x)}{r(x)-\proxy(x)}\big].
\end{aligned}
\end{equation*}
Exactly as in the proof of \Cref{prop:dro-certificate},
\begin{equation*}
\begin{aligned}
&\max_{r\in\U^{(1)}_\varepsilon}
\{J_r(\beta)-J_r(\pi)\}\\
&\qquad =
J_{\proxy}(\beta)-J_{\proxy}(\pi)
+\varepsilon C_{\beta,\pi}.
\end{aligned}
\end{equation*}
Taking the maximum over $\beta\in\Pi_\rho$ yields \eqref{eq:drro-certificate}. The interchange between the maximization over $\beta$ and the maximization over $r$ is valid because the objective is simply maximized over the Cartesian product $\Pi_\rho\times\U^{(1)}_\varepsilon$.
\end{proof}

\begin{proof}[Proof of \Cref{prop:drro-regret}]
Because $\true\in\U^{(1)}_\varepsilon$ and $\pi_\rho^\star$ is locally optimal under $\true$ over $\Pi_\rho$,
\[
\mathcal R(\pi)
\coloneqq
\max_{r\in\U^{(1)}_\varepsilon}
\Big[\max_{\beta\in\Pi_\rho}J_r(\beta)-J_r(\pi)\Big].
\]
\begin{align}
J_{\true}(\pi_\rho^\star)-J_{\true}(\pi_\rho^{\mathrm{DRRO}})
&\le
\mathcal R(\pi_\rho^{\mathrm{DRRO}})\nonumber\\
&\le
\mathcal R(\pi_\rho^\star).
\label{eq:app-drro-first}
\end{align}
The second inequality uses the optimality of $\pi_\rho^{\mathrm{DRRO}}$.

Now fix any $\beta\in\Pi_\rho$. If $C_{\infty,\mathrm{rel}}(\beta,\pi_\rho^\star;\mu,\mathcal D)=0$, then $z_\beta(x)=z_{\pi_\rho^\star}(x)$ for $\mathcal D$-almost every prompt because $\mu_x$ has full support, and hence
\begin{equation*}
J_{\proxy}(\beta)-J_{\proxy}(\pi_\rho^\star)
=
J_{\true}(\beta)-J_{\true}(\pi_\rho^\star)
\le
0.
\end{equation*}
Otherwise, write $C_\gamma\coloneqq C_{\infty,\mathrm{rel}}(\gamma,\pi_\rho^\star;\mu,\mathcal D)$. By the tightness argument in \Cref{prop:drro-certificate}, for any $\eta>0$ there exists $r_\beta^\eta\in\U^{(1)}_\varepsilon$ such that
\begin{align*}
&J_{\proxy}(\beta)-J_{\proxy}(\pi_\rho^\star)
+\varepsilon C_\beta
\le
J_{r_\beta^\eta}(\beta)-J_{r_\beta^\eta}(\pi_\rho^\star)+\eta.
\end{align*}
Choose $\widetilde\beta\in\argmax_{\gamma\in\Pi_\rho}J_{r_\beta^\eta}(\gamma)$. Then $\widetilde\beta\in\Pi_\rho^\star(\U^{(1)}_\varepsilon)$ and
\begin{equation}
J_{\proxy}(\beta)-J_{\proxy}(\pi_\rho^\star)
+\varepsilon C_\beta
\le\ 
J_{\proxy}(\widetilde\beta)-J_{\proxy}(\pi_\rho^\star)
+\varepsilon C_{\widetilde\beta}
+\eta.
\label{eq:app-restrict-opt}
\end{equation}
Letting $\eta\downarrow0$ shows that the maximum in \Cref{prop:drro-certificate} may be restricted to $\Pi_\rho^\star(\U^{(1)}_\varepsilon)$ when $\pi=\pi_\rho^\star$.

For any $\beta\in\Pi_\rho^\star(\U^{(1)}_\varepsilon)$,
\begin{equation*}
\begin{aligned}
J_{\proxy}(\beta)-J_{\proxy}(\pi_\rho^\star)
&=
\big(J_{\proxy}(\beta)-J_{\true}(\beta)\big)
-\big(J_{\proxy}(\pi_\rho^\star)-J_{\true}(\pi_\rho^\star)\big)
+\big(J_{\true}(\beta)-J_{\true}(\pi_\rho^\star)\big)\\
&\le
\varepsilon C_\beta,
\end{aligned}
\end{equation*}
because $J_{\true}(\beta)-J_{\true}(\pi_\rho^\star)\le 0$ and the pairwise estimation error is bounded by the same weighted $\ell_1$--$\ell_\infty$ argument used above. Combining this with \eqref{eq:app-drro-first}, \Cref{prop:drro-certificate}, and \eqref{eq:app-restrict-opt} gives the following bound. Write $\mathcal B\coloneqq\Pi_\rho^\star(\U^{(1)}_\varepsilon)$.
\begin{equation*}
\begin{aligned}
J_{\true}(\pi_\rho^\star)-J_{\true}(\pi_\rho^{\mathrm{DRRO}})
&\le
\max_{\beta\in\mathcal B}
\left\{
J_{\proxy}(\beta)-J_{\proxy}(\pi_\rho^\star)
+\varepsilon C_\beta
\right\}\\
&\le
2\varepsilon
\max_{\beta\in\mathcal B} C_\beta\\
&=
2\varepsilon\mathcal C_{\infty,\mathrm{rel}}
(\U^{(1)}_\varepsilon,\pi_\rho^\star;\mu,\mathcal D).
\end{aligned}
\end{equation*}
\end{proof}

\end{document}